\documentclass[letterpaper]{article}

\usepackage[utf8]{inputenc}
\usepackage[T1]{fontenc}
\usepackage[letterpaper,margin=1in]{geometry}
\usepackage{microtype}
\usepackage[hyphens]{url}
\usepackage{graphicx}
\usepackage{caption}

\usepackage{amsmath}
\usepackage{amssymb}
\usepackage{amsthm}
\usepackage{mathtools}
\usepackage{tikz}
\usetikzlibrary{positioning,arrows.meta}

\usepackage{natbib}
\usepackage{hyperref}
\hypersetup{colorlinks=true,linkcolor={black!70!blue},citecolor={black!70!blue},urlcolor={black!70!blue}}
\usepackage[capitalize,noabbrev]{cleveref}

\newcommand{\RR}{ \mathbb{R} }
\newcommand{\Set}[1]{\left\{ #1 \right\}}
\newcommand{\inner}[2]{\left< #1 , #2 \right>}
\newcommand{\Exp}[1]{ \mathbb{E} #1}
\newcommand{\norm}[1]{\left\|#1\right\|}
\newcommand{\eps}{\varepsilon}
\newcommand{\cond}{\mid}
\DeclareMathOperator*{\argmax}{arg\,max}

\newtheorem{theorem}{Theorem}[section]

\newtheorem{definition}[theorem]{Definition}
\newtheorem{lemma}[theorem]{Lemma}

\newtheorem{remark}[theorem]{Remark}
\newtheorem{proposition}[theorem]{Proposition}

\newlength{\figwidth}
\newcommand{\Vphi}{V_{\phi}}
\newcommand{\Vphip}{V'_{\phi}}
\newcommand{\HS}{\mathrm{HS}}

\newcommand{\mcH}{\mathcal{H}}
\newcommand{\mcX}{\mathcal{X}}
\newcommand{\mcY}{\mathcal{Y}}
\newcommand{\mcE}{\mathcal{E}}

\newcommand{\mcZ}{\mathcal{Z}}
\newcommand{\mcN}{\mathcal{N}}
\newcommand{\mcG}{\mathcal{G}}

\newcommand{\dS}{d_{S}}
\newcommand{\dSb}{d_{S^\ast}}
\newcommand{\dJ}{d_{J}}
\newcommand{\Jf}{J_{\mathrm f}}
\newcommand{\Jb}{J_{\mathrm b}}

\title{Score the Algebra, Not the Span:\\Dimension Reduction for Transfer Operator Models of Dynamical Systems}
\author{
  Mark Kozdoba$^{1}$\\
  Technion, IIT
  \and
  Shie Mannor\\
  Technion, IIT and NVIDIA
}
\date{}

\begin{document}

\maketitle
\footnotetext[1]{\texttt{markk@technion.ac.il}}

\begin{abstract}
Dimension reduction for dynamical systems is standard practice, and the standard
route is spectral: model the transfer (Koopman) operator by its leading modes. We show
that on systems assembled from several weakly interacting components --- a structure common
in physical and biological settings --- this may either require an exponential number of
modes, or drop an entire component: the component is absent from the model rather than
modeled coarsely, and no function of it can be predicted at any accuracy. We call this
\emph{linear masking}.

The cause is that a rank-based model pays one coordinate per mode. We propose to score
instead the $\sigma$-algebra the coordinates generate, so that products and powers come
free and a component's cost is governed only by its generators rather than by all its
interactions. The criterion is a
$\chi^2$-divergence between the embedded present and future, and it carries a budget
guarantee: twice the intrinsic dimension of the dynamics is enough coordinates for an
embedding whose algebra carries the operator's entire spectrum, with its full infinite rank. 

In variational form the criterion admits off-the-shelf estimators, and restricting its
critic to the bilinear class returns the VAMP score on the span, so rank-based methods are
one end of the same family. 
We demonstrate the proposed objective on a composite of published benchmark systems. We exhibit examples where the rank-based methods completely miss the
masked components at all ranks $k<100$, while ten algebra coordinates recover all of
them. In addition, the resulting algebra representation supports predicting the masked
components from few labels, while direct regression from the high-dimensional
observation or from  the VAMP features fail.

\end{abstract}

\section{Introduction}
\label{sec:intro}

Learning a model of a dynamical system from observed trajectories is a basic machine learning
task. Instead of modeling the dynamics on the state space $\mcX$ directly, one standard approach
is to work with the linear operator $T$ that the dynamics induces on the space of functions $f$
over $\mcX$, carrying a function of the future state to its conditional expectation given the
present. In deterministic systems this is known as the Koopman operator
\citep{brunton2022modern}, and it is simply the transition operator in the
Markov setting \citep{meyn2009markov}.

For many physical systems the eigenfunctions of $T$ are known in closed form, and serve as the
basis of the analysis of the dynamics. However, when the system
is given only through sampled trajectories, the eigenfunctions have to be learned, and 
the standard target to be approximated is
a finite-dimensional invariant subspace. For instance, the extended dynamic mode decomposition (EDMD) fits one inside a
chosen dictionary \citep{williams2015data}, 
while in the stochastic and non-reversible setting,
one of the most commonly used objectives for learning
the leading $k$-dimensional singular subspace is
VAMP \citep{wu2020variational}. Its neural form, VAMPnets
\citep{mardt2018vampnets}, learns $k$ functions that realize the best rank-$k$
approximation to $T$.

\paragraph{Multicomponent systems and linear masking.}
Many systems of interest are multicomponent by nature, and biological ones especially so. For instance, a
biomolecular complex is well known to decompose into domains whose internal kinetics are fast
and whose mutual coupling is weak. This near-product structure is then put to use: decomposing and modeling such
systems domain by domain rather than jointly is an established strategy
\citep{hempel2021independent,mardt2022deep}.
However,
finding such a decomposition is itself hard. It presupposes both the number of components and the number of
states allotted to each, and it applies only where those components are uncoupled or weakly coupled, a restriction that is stated explicitly in the literature \citep{mardt2022deep}.  

Here we ask instead whether such systems can be modeled directly, without the need to decompose the system. 

As mentioned above, a common approach to direct modeling of dynamical systems is by the
rank-$k$ approximation of the transfer operator $T$. However, the product structure above may severely limit the approximability of a system by a low-rank operator.
The reason is a counting one: such a model spends one coordinate on each singular
function it retains. When the
components are independent, the singular functions of the joint operator are products of the singular functions of the components, and similarly the singular values are products of the components' singular values. This implies in turn that the number of
modes above a given level may grow exponentially in the number of components.

We note that the most problematic manifestation of this phenomenon happens
when there are spectral gaps between the components themselves. For instance, if there
are a few slow components (large singular values), and a faster component (smaller leading singular value), then many products of the slow components' values will be higher than the leading value of the fast component.
In \Cref{sec:experiments} we provide an example of a system with eight components, in
which a rank $k=100$ representation contains no mode involving \emph{six} of them, and
therefore carries no information about those six at all --- they are not approximated
coarsely, they are absent, while each of them is individually predictable. This is
striking, since the whole system is intrinsically ten-dimensional, and one might
expect a $100$-dimensional representation to carry all the information about it.
We call this phenomenon \emph{linear masking}.

\paragraph{Approximating by an algebra.}
The underlying reason for linear masking is that rank based models must assign 
a separate coordinate to every nonlinear interaction between existing singular
functions (such as their products), even though such an interaction is already
computable from the coordinates the model holds. We therefore ask for an approximation
that is credited for everything computable from its coordinates, rather than one that
must buy each combination separately. 

We formalize this intuition by looking for an embedding
$\phi=(\phi_1,\ldots,\phi_m)$ that generates the largest $\sigma$-algebra, as measured
by the amount of the spectrum of $T$ it captures.
That is, consider the space $V_{\phi}$ of functions of $x$ computable from the
representation $\phi(x)$ alone. Then, roughly speaking, we are interested in the
$\phi$ for which the restriction of $T$ to $V_{\phi}$ has the largest
Hilbert--Schmidt norm.
Although $\phi$ is finite dimensional, the space $V_{\phi}$ is typically not, and it
contains all the interactions between the coordinates $\phi_i$.
For comparison, replacing $V_{\phi}$ by the span
$U_{\phi}=\operatorname{span}\Set{\phi_1,\ldots,\phi_m}$ recovers the rank-$m$
objective of VAMPnets.

In \Cref{thm:diffusion-equivalence} we show that a finite budget always suffices: the
number of coordinates whose algebra carries the whole of $T$ is bounded by twice the
\emph{intrinsic dimension} of the dynamics, a quantity that agrees with the number of
degrees of freedom of the system in the cases of interest.
The cost of a faithful representation is thus set by the
degrees of freedom, rather than by the number of significant modes,
which is what avoids the masking pathology described above.

It is also worth noting that a wide literature reads a dimension off a knee in the
spectrum of $T$
\citep{noe2013variational,dsilva2018parsimonious,vonlindheim2018intrinsic,talmon2015intrinsic}.
Such spectrum-based notions of dimension can be arbitrarily large on a system with
few degrees of freedom, as the example above shows.

\paragraph{Practical modeling.}
While optimizing $\phi$ to score the algebra may \emph{a priori} appear complex, the
restricted Hilbert--Schmidt norm above is in fact exactly the
$\chi^2$ dependence between the embedded present and future. Both it and the map $\phi$
attaining it can therefore be computed with off-the-shelf $\chi^2$
estimators \citep{sugiyama2012density,kanamori2009least,nguyen2010estimating}.

\paragraph{Experiments.} We consider a ten-dimensional system with eight
components assembled from published benchmark systems, some slow and some fast, and a
parameter $\kappa$ controlling the amount of dependence between them. Its trajectories
are observed either through the native ten-dimensional state, or through a
high-dimensional nonlinear warp, as is more common in practice.

We demonstrate empirically the phenomena that the theory predicts.
VAMPnets trained to convergence recover nothing of the six fast components at any rank
below the product count of about one hundred, although each is individually
predictable from the observations. This holds at every setting of $\kappa$ and under
both observation models. Maximizing the algebra criterion over ten
coordinates, one per degree of freedom, recovers every component in all of these
settings.

The value of a learned dimension reduction often lies in the label budget: trained on
abundant \emph{unlabeled} transition pairs, a representation can support predicting the
masked components from few labels. We show that the algebra objective produces
representations where this is possible, while direct regression from the
observations and rank-based representations both fail.

\section{Related Work}
\label{sec:related}

The standard low-rank program scores the linear \emph{span} of its coordinates, and so
pays one coordinate per significant mode. The
best rank-$k$ span is the leading singular subspace of the transfer operator, which for
a composition operator is the Koopman operator of the dynamics. Taking those leading
functions as coordinates is the program of Koopman operator theory
\citep{brunton2022modern}, realized empirically by dynamic mode decomposition and its
extended and kernel variants \citep{williams2015data,williams2015kernel}. In the
stochastic, non-reversible setting the variational score for that subspace is VAMP
\citep{wu2020variational}, and its neural form, scoring the span of a learned
$k$-dimensional feature map, is VAMPnets \citep{mardt2018vampnets}. The current neural
mode learners target the same top-$k$ subspace
\citep{jeong2025efficient,kostic2024learning,deng2022neuralef,pfau2019spectral}, as does
the kernel and RKHS operator-regression line
\citep{klus2020eigendecompositions,kostic2022learning}, where the rank-$k$ truncation is
replaced by a restriction on the RKHS norm. On a multicomponent system all of them
over-count the intrinsic dimension by an unbounded factor, since the modes they enumerate
are products and harmonics of a few generators. VAMPnet is our primary benchmark
(\Cref{sec:experiments}).

Closest in machinery,
\citet{turri2025self} train a feature map through the same least-squares density-ratio
functional we estimate with, with a \emph{bilinear} critic whose
optimal value they identify with the VAMP-2 score --- the span once more, now reached
variationally. Our one move is to leave that critic unrestricted, which scores the
algebra the coordinates generate and turns the mode count into an intrinsic dimension
(\Cref{subsec:objective-connections}).

\citet{bittracher2018transition} likewise seek a mapping that preserves the dominant
spectral subspace, and reach it through the embedding machinery our
\Cref{thm:diffusion-equivalence} also uses, over the same $\chi^2$ geometry of predictive
laws. Their guarantee is a dominant-mode one by construction: the recovered
eigenfunctions carry an error bounded by $\eps/|\lambda_i|$, which diverges as
$|\lambda_i|\to0$, so it is silent on the fast modes a rank budget discards, and they
assume reversibility, which we do not. Computational mechanics reaches for the same
object conceptually: causal states are the classes of a predictive-equivalence
$\sigma$-algebra \citep{shalizi2001computational}, constructed in practice through kernel
mean embeddings of the predictive laws \citep{brodu2020discovering}, and so anchored to
the kernel's view of the observed state.

Additional literature notes, on component identification and on autoencoder
methods, are in \Cref{sec:additional-lit}.

\section{Preliminaries}
\label{sec:prelim}

\paragraph{The coupling.} We work with a single state space $\mcX$ carrying a
probability measure $\mu$ together with a Markov kernel $p(\cdot\cond x)$, the one-step
law of the future given the present. Let $X\sim\mu$ be the present state,
$X'\sim p(\cdot\cond X)$ the coupled future, and $\mu'$ the induced future marginal,
$d\mu'(x')=\int p(x'\cond x)\,d\mu(x)$. We assume neither stationarity ($\mu'$ may differ
from $\mu$) nor reversibility. The pair $(X,X')$ has joint law $\bar\mu$ on
$\mcX\times\mcX$, and this joint law is the sole input --- every object below is a
functional of $\bar\mu$ alone. A \emph{coupling} is thus any pair of random variables,
with the present/future reading as its running instance.

\paragraph{The transfer operator.} The coupling is carried by its
conditional-expectation (transfer) operator $T:L^2(\mcX,\mu')\to L^2(\mcX,\mu)$,
\begin{equation}
\label{eq:cond-operator}
  (Tf)(x)=\Exp{f(X')\cond X=x},
\end{equation}
which sends a future observable to its best present prediction. We also write $T=T_{X,X'}$, when the underlying coupling needs to be specified explicitly. The adjoint of $T$, 
$T^\ast:L^2(\mcX,\mu)\to L^2(\mcX,\mu')$ is the backward conditional expectation
$(T^\ast g)(x')=\Exp{g(X)\cond X'=x'}$, and constants are fixed,
$T\mathbf 1=\mathbf 1=T^\ast\mathbf 1$.

\paragraph{Regularity and the singular system.} The single standing assumption in this paper is that
$T$ is \emph{Hilbert--Schmidt}, $\operatorname{tr}(T^\ast T)<\infty$. This is also the standing
assumption of the variational Koopman literature, where it underwrites the singular value
decomposition and likewise excludes deterministic dynamics \citep{wu2020variational}. A
Hilbert--Schmidt operator is compact, so $T$ admits a singular value decomposition
$(\sigma_j,\phi_j,\psi_j)_{j\ge0}$ \citep{hsing2015theoretical},
\begin{equation}
\label{eq:svd}
  T\psi_j=\sigma_j\phi_j,\qquad T^\ast\phi_j=\sigma_j\psi_j,
\end{equation}
with singular values $1=\sigma_0\ge\sigma_1\ge\cdots\to0$, present functions
$\{\phi_j\}$ orthonormal in $L^2(\mu)$, future functions $\{\psi_j\}$ orthonormal in
$L^2(\mu')$, and the trivial leading triple $\sigma_0=1$, $\phi_0=\psi_0=\mathbf 1$.
The trace is the sum of the squared singular values, and is the squared
Hilbert--Schmidt norm:
$\operatorname{tr}(T^\ast T)=\norm{T}_{\HS}^2=\sum_{j\ge0}\sigma_j^2$.

In addition, the trace is related to the $\chi^2$ dependence between $X$ and $X'$ ---
the $\chi^2$ divergence of the joint law from the independent product of the
marginals, written $\chi^2(X;X')$. Let $P$ be the joint law of $(X,X')$ and
$r=dP/d(\mu\otimes\mu')$ its density ratio against the product of the marginals.
Then the transfer operator
$T=T_{X,X'}$ is the integral operator with kernel $r$, $(Tf)(x)=\int r(x,x')f(x')\,d\mu'(x')$ and we have
\begin{equation}
\label{eq:hs-chi2-prelim}
\begin{aligned}
  \norm{T}_{\HS}^2 &=\int r^2\,d(\mu\otimes\mu')=1+\chi^2(X;X').
\end{aligned}
\end{equation}
We thus use $\norm{T}_{\HS}^2$ and $1+\chi^2(X;X')$ interchangeably throughout the
paper.

\section{Framework and Results}
\label{sec:framework}

\subsection{General Embedding Results}
\label{subsec:general-embedding}

\paragraph{Induced couplings.} Let $\phi:\mcX\to\mcY$ be a measurable feature map. With $Y=\phi(X)$ and $Y'=\phi(X')$ the pair $(Y,Y')$ is again a coupling, the
one \emph{induced} by $\phi$, and it carries its own transfer operator $T_{Y,Y'}$.

\paragraph{The algebra objective.} We will be interested in embeddings $\phi$ that
maximize the size of the induced operator, $T_{\phi}:=T_{Y,Y'}$. By
\eqref{eq:hs-chi2-prelim}, applied to the induced coupling,
\begin{equation}
\label{eq:objective}
  \norm{T_{\phi}}_{\HS}^2
  \;=\; 1+\chi^2\!\big(\phi(X);\phi(X')\big),
\end{equation}
the predictive $\chi^2$ dependence --- the total principal inertia --- of the embedded
present and future. We call \eqref{eq:objective} the \emph{algebra objective} because
its value depends on $\phi$ only through the $\sigma$-algebra $\sigma(\phi)$ the
coordinates generate:\footnote{A function $g$ is $\sigma(\phi)$-measurable iff it is of
the form $g=\rho\circ\phi$ for some measurable $\rho:\mcY\to\RR$ --- iff it is a
function of the coordinates.} any embedding computing
the same information receives the same value, and every product, power, or other
measurable function of the coordinates is credited for free.
\Cref{lem:induced-projection} in \Cref{sec:proofs} states the relation between
\eqref{eq:objective} and the space $V_{\phi}$ of \Cref{sec:intro} precisely: $T_\phi$
is the compression of $T$ onto $V_{\phi}$ and its future counterpart.
\paragraph{The class of maps, and the optimum over it.} The objective is defined for every
measurable $\phi$, but what is \emph{attainable} depends on which maps one is willing to
search. We therefore carry the class as an explicit parameter: $\Phi$ is a set of
measurable maps $\mcX\to\RR^m$, and $\phi^\ast$ denotes a maximizer of the objective over it,
\begin{equation}
\label{eq:argmax}
  \phi^\ast\;\in\;\argmax_{\phi\in\Phi}\;\norm{T_{\phi}}_{\HS}^2 ,
\end{equation}
attaining the value $\norm{T_{\phi^\ast}}_{\HS}^2$. Since the criterion sees $\phi$ only
through the algebra it generates, $\phi^\ast$ is determined only up to that algebra,
which is the object actually selected. Two properties hold for every $\Phi$: the value never exceeds
$\norm{T}_{\HS}^2=1+\chi^2(X;X')$, by the data-processing inequality, and it is
non-decreasing as $\Phi$ grows, in particular when a coordinate is adjoined to every map in
it. Three classes appear below --- all measurable maps, the Lipschitz maps that
\Cref{thm:diffusion-equivalence} produces, and the networks on the observed coordinates
that we actually fit.

\paragraph{Remeasuring the state.} Suppose the state is observed through a map
$h:\mcX\to\mcZ$, so that models must be built from $Z=h(X)$ rather than from $X$
itself. For a general $h$ this loses dependence, by data processing. When $h$ is
invertible, however, nothing is lost: the induced coupling $(h(X),h(X'))$ is the
original coupling relabeled, and every embedding $\phi$ of the state translates to
the embedding $\tilde{\phi} = \phi\circ h^{-1}$ of the observation --- the same coordinates, now
computed from $Z$, provided the class of maps on $\mcZ$ is rich enough to contain 
$\tilde{\phi}$. We return to this in the experiments (\Cref{subsec:exp-warped}).
See also \Cref{sec:invariance-app}.

\paragraph{A budget of $m$ coordinates.}
With $\Phi$ a family of maps $\mcX \to \RR^m$, how large should the dimension $m$ be?
We first note that this depends on the complexity of $\Phi$: a single coordinate,
$m=1$, made complex enough, can be injective, and its generated algebra then carries
everything, so the budget $m$ by itself is not an obstacle. The real question is ---
what $m$ suffices with maps of controlled complexity? For instance, what are we
guaranteed for a system with smooth dynamics in a $d$-dimensional state space and with
reasonably behaved maps $\phi$?

To answer this question, we first define the \emph{intrinsic dimension} of a system in
terms of quantities governed only by the coupling $(X,X')$ itself. Specifically, we set
$d^\ast$ to be the box dimension \citep{falconer2014fractal} of the state space $\mcX$
under the predictive diffusion metric of \citet{coifman2006diffusion}, in which two
states are close when their predictive laws are close. This definition parallels classical
constructions in dynamics, for instance 
\citep{farmer1983dimension},
\citep{sauer1991embedology} --- with the attractor replaced by the coupling's predictive
geometry, as in \citet{bittracher2018transition}. The intrinsic dimension behaves as
expected on smooth systems: on a compact $d$-dimensional state space with smooth
enough, nondegenerate transitions we have $d^\ast\le d$
(\Cref{prop:dstar-bound}), and the system of \Cref{sec:experiments} has
$d^\ast=10$. The full details of the construction are given in
\Cref{sec:diffusion-dim-app}.

In the next theorem we then show that $m = \lfloor 2d^\ast \rfloor + 1$ always
suffices. This holds no matter how many significant modes the operator carries: the
guaranteed budget is set by the dimension of the dynamics, not by the size of its
spectrum.

\begin{theorem}[A finite budget suffices]
\label{thm:diffusion-equivalence}
Assume the coupling is Hilbert--Schmidt with finite intrinsic dimension $d^\ast<\infty$.
Then for every integer $m>2d^\ast$ there exists a map
$\phi:\mcX\to\RR^m$, Lipschitz with respect to the predictive metric, whose generated
algebra carries the whole operator,
\begin{equation}
\label{eq:bridge-bracket}
  \norm{T_{\phi}}_{\HS}^2 \;=\; \norm{T}_{\HS}^2 \;=\; 1+\chi^2(X;X') .
\end{equation}
In particular $\lfloor2d^\ast\rfloor+1$ coordinates suffice.
\end{theorem}

The proof sketch and the full proof are in \Cref{sec:diffusion-dim-app}. The proof
builds on general embedology results in Banach spaces \citep{hunt1999regularity}, and
is essentially an existence proof, exploiting a certain kind of random linear
mappings, and does not otherwise provide a recipe for constructing the map. Maximizing
\eqref{eq:objective} may be viewed as the learning of such a mapping.

\subsection{Connection with Other Objectives}
\label{subsec:objective-connections}

\paragraph{Span versus algebra: what is scored.} The classical rank-based objective,
VAMP \citep{wu2020variational} --- the objective of VAMPnets \citep{mardt2018vampnets}
--- searches for the map $\phi(x) = (\phi_1(x),\dots,\phi_m(x)) \in \RR^m$ for which
$T$ restricted to the linear \emph{span} of the coordinate functions,
$\operatorname{span}\Set{\phi_1,\dots,\phi_m}$, has the largest norm. In
\Cref{lem:induced-projection}
(\Cref{sec:proofs}) we show that the objective \eqref{eq:objective} is precisely this
criterion with the span replaced by the whole $\sigma$-algebra the coordinates
generate --- products, powers, and every other measurable function of the $\phi_i$
included. The span credits only the coordinates themselves --- every product and power
must be bought with a further coordinate --- while the algebra credits them for free.
We take the larger algebra, which allows the map $\phi$ to extract more information at
the same budget.

\paragraph{Estimation via the variational form.}
Next, one way to estimate a $\chi^2$ divergence is via its variational form. For the
coupling $(X,X')$, with joint law $P$ and independent product of marginals
$\mu\otimes\mu'$ as in \Cref{sec:prelim}, the $\chi^2$ is a supremum over functions
$g\in L^2(\mu\otimes\mu')$ \emph{of the pair}, called critics:
\begin{equation}
\label{eq:variational-body}
  \sup_{g}\;\Big(2\!\int\! g\,dP-\int\! g^2\,d(\mu\otimes\mu')\Big)
  \;=\;1+\chi^2(X;X')
\end{equation}
\citep{kanamori2009least,nguyen2010estimating}. Proofs and attribution
are in \Cref{sec:lsif-app}.

\paragraph{Bilinear versus full critics: what is estimated.} Applied to the induced
coupling $(\phi(X),\phi(X'))$, the supremum in \eqref{eq:variational-body} computes the
$\chi^2$ in \eqref{eq:objective}, with critics $g(y,y')$ now functions of the embedded
pair. We use this approach for the experiments in \Cref{sec:experiments}.

Notably, the critic class is where the span--algebra distinction re-enters: restricting
the critics $g$ in \eqref{eq:variational-body} to a smaller class captures a smaller part of the dependence. Over the full
critic class the supremum is the entire $1+\chi^2(\phi(X);\phi(X'))$.
Restricting to
\emph{bilinear} critics --- functions affine in each argument,
$g(y,y')=\alpha+a^{\top}y+b^{\top}y'+y^{\top}\!B\,y'$ --- admits only the coordinates
themselves, and the supremum collapses to
exactly the VAMP-2 span score of $\phi$ (see 
\citealp{turri2025self,haochen2021provable}). 

Finally, we note that the fact that the $\chi^2$
captures the full dependence of the coupling is well known. The novelty in this paper is in
maximizing $\phi$ against this criterion rather than the bilinear one, in the motivation
for doing so, and in its implications.

\section{Experiments}
\label{sec:experiments}

We study masking and the algebra objective on a multicomponent diffusion assembled
from published benchmark systems, with a coupling dial $\kappa$ controlling the
interdependence of its components. The system is intrinsically ten-dimensional, and
its components are of two kinds: two \emph{slow} ones with rich spectra, and six
\emph{fast} ones whose modes lie below them. We
show that span methods mask its fast components up to large rank --- through $k=80$,
with the transition at $k\approx100$, the product-mode count --- while
a ten-dimensional maximizer of the algebra objective recovers every component
(\Cref{subsec:exp-masking}). Both objectives are invariant under a
large class of reparametrizations (\Cref{subsec:general-embedding}), so constructing the
representations on the native state is no loss of generality; we test this
directly by rerunning the program through a high-dimensional, strongly
non-isometric warp of the state (\Cref{subsec:exp-warped}), where an interesting
optimization phenomenon also appears. Finally we ask what a representation buys
over regressing each target directly from the data: learned from unlabeled
observations, the embedding supports prediction from far fewer labeled samples
(\Cref{subsec:exp-fewshot}).

\subsection{The Composite System}
\label{subsec:exp-system}

\paragraph{Components.} The system is built from \emph{slow} and \emph{fast}
components, which play opposite roles in the masking geometry: the slow ones carry
rich spectra, so their products supply the leading modes of the joint operator,
while each fast one contributes a single mode, placed below those products.
The slow part consists of $r=2$ independent copies of the
lemon-slice diffusion --- overdamped Langevin dynamics in the circular multi-well
potential of \citet{bittracher2018transition}, as used by
\citet{klus2020eigendecompositions} and named in \citet{klus2020data} ---
\begin{equation*}
  V_{\mathrm{cool}}(x,y)=\cos\!\big(a\,\mathrm{atan2}(y,x)\big)
  +10\big(\sqrt{x^2+y^2}-1\big)^2,
\end{equation*}
with $a=10$ wells and diffusion coefficient $D=0.25$. The fast part is a
six-dimensional hypercube of independent double-well bits,
$V_{\mathrm{hot}}(x)=2(x^2-1)^2$ per coordinate at $D=0.75$, the benchmark family
of iVAMPnets \citep{mardt2022deep}. We call the ring components \emph{cools} and
the bits \emph{hots}, and the state is $2\times 2 + 6 \times 1 = 10$ dimensional. 
All experiments use transition pairs at lag $\tau=2$. 

\paragraph{Coupling.} A dial $\kappa\ge0$ couples the components: a Kuramoto
torque $\pm\kappa\sin(\theta_1-\theta_0)$ between the two cool component angles, and a
directed tilt $-\kappa\cos(\theta_0-\psi_j)\,x_j$ on bit $j$, with phases
$\psi_j=2\pi j/6$, so cool$_0$ modulates every bit while no bit feeds back
(\Cref{fig:coupling}). At
$\kappa=0$ the components are independent and the cool marginals are exact
nearest-neighbor ring walks; we use $\kappa\in\{0,0.5,1,2\}$, chosen so that
every component remains live and autonomous at every setting.

\begin{figure}[t]
\centering
\resizebox{\figwidth}{!}{%
\begin{tikzpicture}[>={Stealth}, thick,
  cool/.style={circle, draw, minimum size=1.15cm, inner sep=0pt},
  bit/.style={rectangle, rounded corners=2pt, draw, minimum width=0.9cm,
              minimum height=0.65cm, inner sep=1pt}]
  \node[cool] (c0) at (1.0,0) {$\theta_0$};
  \node[cool] (c1) at (5.4,0) {$\theta_1$};
  \node[above=1pt of c0] {\small cool$_0$ (ring, $a=10$)};
  \node[above=1pt of c1] {\small cool$_1$ (ring, $a=10$)};
  \draw[<->] (c0) -- (c1) node[midway, above] {\small $\pm\kappa\sin(\theta_1-\theta_0)$};
  \foreach \j in {0,...,5} {
    \node[bit] (b\j) at ({-1.9+\j*1.35}, -2.5) {$x_{\j}$};
    \draw[->] (c0) -- (b\j);
  }
  \node[anchor=west] at (-2.6,-1.35) {\small $-\kappa\cos(\theta_0-\psi_j)\,x_j$};
\end{tikzpicture}}
\caption{The coupling at dial $\kappa$: a symmetric torque between the two rings, and
a phase-staggered tilt from cool$_0$ on each bit, with no feedback. At $\kappa=0$ every
arrow vanishes.}
\label{fig:coupling}
\end{figure}
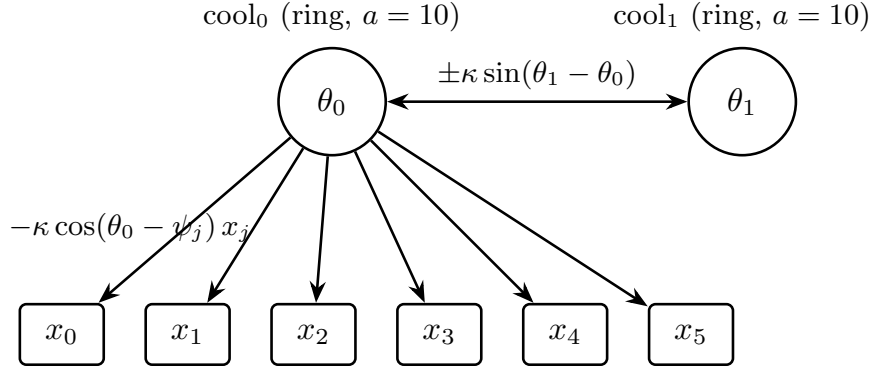

\paragraph{Spectrum, and why masking is predicted.} The leading singular values of a single cool component can be approximated by
those of a cyclic walk on its wells \citep{sarich2010approximation,prinz2011markov},
$\lambda_i=1-p(1-\cos(2\pi i/a))$ at hop probability $p=0.080$
\citep{levin2017markov}, yielding values $0.985/0.945/0.895$ for the leading nontrivial modes, verified in addition 
by a Markov state model approximation on simulated data (\Cref{fig:spectrum}). Each
bit contributes a single mode at $\sigma\approx0.55$. At
$\kappa=0$ the joint singular values are products across components, so the whole joint
spectrum follows from these marginals by multiplication, with no estimation involved.
\Cref{fig:spectrum} plots it: the two cools alone generate $a^{r}=100$ product modes,
and the \emph{entire} hot block falls below them, the lowest cool product sitting at
$(1-2p)^2\approx0.71$ against $\sigma_1^{\mathrm{hot}}=0.561$ for the highest mode that
involves a bit. The joint rank saturates at $100$. A rank-$k$ span optimum therefore fills with cool products
and excludes every bit until $k$ reaches the product count --- roughly one
hundred coordinates for a system whose intrinsic dimension is ten --- while six
individually predictable components sit outside the span.

\begin{figure}[t]
\centering
\includegraphics[width=\figwidth]{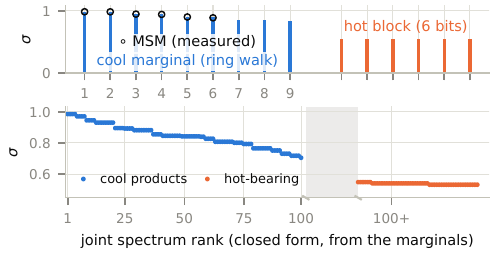}
\caption{The composite system's spectrum. Top: the marginal of one cool, a ring
walk, with the measured Markov state model overlaid, and the six bits as a degenerate
block at $\sigma\approx0.55$. Bottom: the joint spectrum at $\kappa=0$, the product of
the marginals. The rank axis is broken because the crossover is near rank $100$ but not
pinned there, the closed form idealizing each cool as a ten-state ring walk.}
\label{fig:spectrum}
\end{figure}

\subsection{Estimators and Protocol}
\label{subsec:exp-protocol}

\paragraph{Span.} VAMPnets \citep{mardt2018vampnets,wu2020variational} train a
$k$-output network on the VAMP score. We train to convergence under a fixed
protocol: two restarts, checkpoints to $2400$ epochs, restart selection and
snapshotting on held-out VAMP-E. Convergence is not a formality here: at every
setting we tested, an under-trained VAMPnet carries transient bit signal that its
own objective removes as the held-out score improves
(\Cref{subsec:exp-warped}), so a shorter protocol reports spurious recovery. An undersized network produces the same
artifact: unable to build the leading product modes, it spends those coordinates on the
bits instead (\Cref{sec:capacity-app}).

\paragraph{Algebra.} We maximize \eqref{eq:variational-body} by joint ascent over
an $m$-output embedding $\phi$ and an unrestricted MLP critic $g$, with snapshot
early stopping on the held-out objective value $\hat J$. Restricting the same
functional to bilinear critics recovers the span case
(\Cref{subsec:objective-connections}). Run in our own machinery, that restriction is
indistinguishable from VAMPnet on every readout, which confirms the phenomenon through
an independent implementation (\Cref{tab:dial}). Representations train on $n=30{,}000$ unlabeled pairs.
Writing $w^{\ell}$ for a fully connected trunk of $\ell$ hidden layers of $w$
units each, the algebra trunk is $64^2$ on the native observation and $128^2$ on
the warped one --- never wider than the span trunk it is compared against, which
we sweep to $512^2$.

\paragraph{Readout.} Every method is scored the same way: uniform prediction
targets --- each component's future state coordinates, identically for every
component --- regressed from the learned features by one neural readout on
separate splits, reported as per-component test $R^2$. Two reference rows apply that same readout to inputs
that are not learned. The first is the full state. Every representation is a function of
the state, so under a common readout none can predict a target better than the state
itself does: this row is the \emph{ceiling}, the best score any method could reach. The
second withholds the bit coordinates and reads out from the cools alone. This is the
\emph{floor benchmark}: once the components are coupled the cools carry some
information about the bits, so a method clears that level without representing a bit at
all. It is zero at $\kappa=0$ and rises with the coupling.

\subsection{Masking and Recovery}
\label{subsec:exp-masking}

\Cref{fig:kfamily} shows bit recovery against span rank $k$, at the largest trunks we
train ($512^2$), one curve per $\kappa$, with each curve's floor benchmark and ceiling marked, and the algebra objective at $m=10$ as a level marker.
The pattern is the same at every coupling: the span sits \emph{exactly} on its
floor benchmark through $k=80$ --- at $\kappa=0$ that floor is zero and the
converged span returns $-0.01$ at $k=80$ across two restarts, at $\kappa=1$ VAMPnet
reproduces the floor value $0.117$ to the third decimal at both $k=50$ and $k=80$ --- then admits the bits gradually
across $k=100$--$120$, near the product count. We call that rank the \emph{threshold}:
the smallest budget at which a span carries anything of the bits' own content. It does
not move with $\kappa$: coupling raises the floor (cool-borne bit information grows until, at
$\kappa=2$, the cools alone reach $82\%$ of the bits' predictable content) but not the
rank at which the bits' own content enters. \Cref{tab:dial} gives the same comparison
numerically, at one rank per coupling, where the floor and the ceiling are easier to
read off than from the figure: the span reproduces the floor, and the algebra objective
at $m=10$ reaches the ceiling, at every coupling. Cool recovery is uniformly high
($0.91$--$0.96$) for every method at every $k$ and $\kappa$.

\begin{table}[t]
\centering
\small
\setlength{\tabcolsep}{3pt}
\begin{tabular}{lcccc}
\hline
bit recovery ($R^2$) & $\kappa=0$ & $\kappa=0.5$ & $\kappa=1$ & $\kappa=2$ \\
\hline
floor (cools only)   & $0.00$ & $0.02$ & $0.12$ & $0.39$ \\
span, $k=30$         & $0.00$ & $0.02$ & $0.12$ & $0.39$ \\
bilinear critic      & $0.04$ & --- & $0.12$ & --- \\
algebra, $m=10$      & $0.24$--$0.27$ & $0.24$--$0.28$ & $0.26$--$0.32$ & $0.47$ \\
ceiling (raw state)  & $0.26$ & $0.27$ & $0.31$ & $0.48$ \\
\hline
\end{tabular}
\caption{Bit recovery along the coupling dial, mean over the six fast components.
The floor is what the cools alone reveal about them. The bilinear-critic row is our
own objective at $m=10$ with the critic restricted to the span, run at two couplings.
Algebra entries are min--max over three, two, three and two runs.}
\label{tab:dial}
\end{table}

\begin{figure}[t]
\centering
\includegraphics[width=\figwidth]{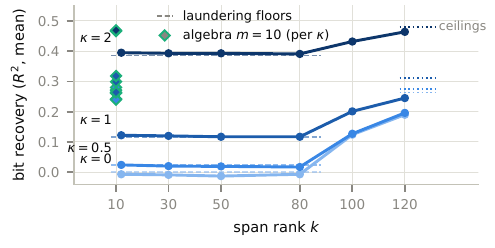}
\caption{Masking along the coupling dial. Bit recovery of the rank-$k$ span
optimum stays on the floor benchmark until $k$ approaches the product-mode
count ($\approx100$), at every coupling strength, while the algebra objective
recovers the bits at $m=10$. The floor rises with $\kappa$ as the cools come to carry more about the
bits; the threshold does not move.}
\label{fig:kfamily}
\end{figure}

\subsection{Warped Observations}
\label{subsec:exp-warped}

Native dynamical coordinates are rarely observed in practice: what is given instead
is a larger set of correlated measurements that describe them, which is why a
dimension reduction is needed at all. We model this situation by observing the state
through a fixed random injective warp $S:\RR^{10}\to\RR^{100}$ --- an orthonormal
injection followed by four random affine coupling layers of a normalizing flow ---
smooth, exactly invertible, and strongly non-isometric (local scale distortion
$\approx7$ measured on data pairs). Thus, in this subsection and the next, every method
receives the $100$-dimensional warped observation as input, in place of the
ten-dimensional native state used above. Since the warp is an invertible remeasurement of
the state, the achievable values of both objectives are unchanged (\emph{Remeasuring
the state}, \Cref{subsec:general-embedding}), so the population picture is unmoved.
The measurements agree: VAMPnet stays masked through $k=50$ and admits the bits across
$k=80$--$120$, as it does natively. What the warp costs is estimation: past the
threshold the best held-out span scores run about a third below native.

\paragraph{An optimization phenomenon, and its repair.} Optimizing the algebra objective
on the warped data by direct joint ascent over $\phi$ and $g$ falls short: it plateaus at
$\hat J\approx93$ against $\approx288$ on the native state, recovering the bits
only partially at full readout ($0.19$--$0.23$ against a $0.26$ ceiling) and weakly in
the few-label regime. The shortfall is not capacity --- enlarging $\phi$ in
width or depth strictly lowers the attained value, and enlarging the critic
leaves cold features' scores unchanged --- and it is not the estimand, since a
diagnostic $\phi$ frozen at a supervised inverse of $S$ scores at the native
level. It is the ascent itself.

The repair is standard unsupervised pretraining \citep{erhan2010why}: an
autoencoder on the unlabeled warped stream, trained on reconstruction alone with
no lagged pairs and no dynamical term (reconstruction $R^2=0.85$, so an imperfect
chart suffices), whose encoder warm-starts the ascent. It sees no dynamics and uses no
labels, so it remains an optimization aid rather than a second model. Started this way,
the ascent restores native-level recovery (bits $0.253$ at full readout, three seeds,
negligible seed variance).

The failure itself is interesting: a $30$k-parameter network with $n=30{,}000$ pairs and
a smooth invertible target is a regime in which direct optimization would be expected to
succeed, and it does not. Our measurements place the difficulty in the joint ascent
rather than in the objective or the architecture, and we think identifying it is a
worthwhile question in its own right.

\paragraph{The objective, not the optimizer, removes the bits.} The same warm
start separates the two objectives cleanly, and it is a neutral starting point
precisely because the chart is dynamics-blind: having never seen a lagged pair,
it carries no preference among components into the comparison, so every
divergence below is attributable to the objective that follows it. Given the
identical autoencoder chart, the free-critic ascent keeps the bits ($0.253$); the bilinear-critic
ascent removes them ($-0.01$); and a VAMPnet whose trunk is initialized from the
chart \emph{starts} with readout-verified bit information ($R^2=0.244$ at its
first checkpoint, at the recovery level of the algebra $\phi$) and trains it
away while its held-out score improves monotonically, reaching $-0.01$ by
convergence (\Cref{fig:warped}). Every convergence trace we recorded, in every setting, shows this
decay of early transient bit signal; the warm-started run certifies that what
decays is real, usable information, removed because the span optimum excludes
it. This is also why the convergence protocol of \Cref{subsec:exp-protocol} is
load-bearing: any under-trained span representation exhibits transient
``recovery'' that vanishes at convergence.

\subsection{Prediction from Few Labels}
\label{subsec:exp-fewshot}

Why learn a representation at all, when any target can be regressed directly
from the raw data? We show that on the warped observation the algebra objective learns
representations that support \emph{few-shot learning}. Representations are trained on
$n=30{,}000$ \emph{unlabeled} pairs, each prediction target is then learned from $300$
noisy labeled samples, and we report test $R^2$ against clean targets, resampled over
twenty independent label draws with all methods paired on identical draws. Five
representations are compared under that protocol: the warped observation itself, a
rank-$k$ span, the autoencoder chart, the algebra objective at $m=10$, and the native
state, which is the ceiling.
\Cref{tab:fewshot} gives the comparison. The algebra representation beats direct
regression on both kinds of component, and beats every span rank on the bits.
Its advantage is also not inherited from the unsupervised warm start that initializes
the ascent: the autoencoder chart alone already improves on direct regression, and
training the algebra objective from that chart improves on it again.

\begin{table}[t]
\centering
\small
\setlength{\tabcolsep}{4pt}
\begin{tabular}{lcc}
\hline
from $300$ labels & cools & bits \\
\hline
direct regression        & $0.40$--$0.44$ & $0.03\pm0.04$ \\
span, any rank $12$--$120$ & ---            & $\le 0$ \\
autoencoder chart alone  & $0.49$         & $0.09\pm0.02$ \\
algebra, $m=10$          & $0.77$--$0.80$ & $0.13\pm0.02$ \\
ceiling (full sample)    & ---            & $0.26$ \\
\hline
\end{tabular}
\caption{Prediction from few labels on the warped observation, mean $\pm$ sd over
twenty paired label draws that resample the labels only, with nothing retrained. Spans
are at or below zero on the bits at every rank we measured, $k=12$ to $120$, including
the ranks whose full-sample readout does recover them.}
\label{tab:fewshot}
\end{table}

Span representations sit at or below zero on the bits at \emph{every} rank
$k=12$--$120$, in every setting and at every capacity we measured --- including
ranks past the threshold whose full-sample readout does recover the bits (native
$512^2$, $k=120$: $0.17$--$0.19$ at full sample, $-0.05$ at $300$ labels even
with clean labels). Admission into the span at full sample does not transfer to
the few-label regime: the information is spread across $\approx100$
coordinates, and $300$ labels cannot reassemble it, while the algebra
objective concentrates it in ten. The scope of the advantage is two-fold.
Dominant structure is recoverable from few labels through either compression at
sufficient rank or the algebra objective. The weak components are recoverable only
through the algebra representation --- at full sample by any route that does
not compress them away, at few labels by nothing else we measured.

\section{Conclusion, Limitations, and Future Work}
\label{sec:conclusion}

We proposed a criterion for reducing a dynamical system to few coordinates: score the
algebra the feature map generates rather than the span of its coordinates, so that a
component costs its generators and not its modes. The objective is a $\chi^2$ dependence
between the embedded present and future, optimized directly by standard density-ratio
estimators, and its variational form places VAMP and our criterion in one family,
separated only by the critic class --- bilinear for the span, unrestricted for the
algebra. Experimentally, span objectives exhibit the motivating failure, linear
masking --- whole components dropped at every rank below the count of interaction
modes, a count exponential in the number of components --- while the algebra
objective recovers them at one coordinate per degree of
freedom, and its representations support prediction from few labels where regression
on the observation fails. These results extend the operator-theoretic program rather
than displace it: what changes is the price of a component --- degrees of freedom
rather than modes.

The warped-observation study (\Cref{subsec:exp-warped}) also yields a finding: from
an unsupervised warm start --- standard pretraining, blind to the dynamics --- the
ascent reaches its native value, and the measurements localize the cold-start gap to
the joint optimization rather than to the objective. We leave this as future work.

Finally, the experimental construction we use is a natural extension of studied
benchmark systems, but the data are simulated rather than real --- the next step is
real data, where the components are not known in advance.
\bibliographystyle{plainnat}
\bibliography{refs}

\clearpage
\appendix
\setcounter{figure}{0}
\renewcommand{\thefigure}{S\arabic{figure}}
\setcounter{table}{0}
\renewcommand{\thetable}{S\arabic{table}}
\section*{Supplementary Material}
\label{app:start}

\section*{Contents}

\begin{description}\setlength{\itemsep}{1pt}
\item[A. The Projection Lemma.] The structural identity the algebra objective rests on:
  the coupling induced by $\phi$ is the two-sided compression of $T$ onto the
  algebra $\phi$ generates.
\item[B. The objective as a $\chi^2$ divergence.] That the algebra objective \emph{is} a
  $\chi^2$ divergence, its variational form and the least-squares estimator we
  use, the critic-class reading that puts span and algebra on one axis, the
  bilinear ablation, and what a held-out value does and does not certify.
\item[C. Additional literature notes.] Dependence-maximization methods and the
  information-bottleneck reductions, component identification, repeated
  eigendirections, and autoencoder models of the dynamics --- several of which
  carry a cost exponential in the number of components, in coordinates or in
  batch size.
\item[D. Additional experiments.] D.1 isolates what the unsupervised warm start
  supplies on its own. D.2 is the capacity study: an underfit span network does not
  build the interaction modes its own objective calls for, and may leak bit
  information in their place.
\item[E. The predictive metric.] The \emph{geometry} of the singular system, as
  against the \emph{size} that B measures --- the $\chi^2$ distance between
  predictive laws, the kernel of $TT^\ast$ and its canonical RKHS, and how the
  feature class imposes a topology of its own.
\item[F. Invariance.] The algebra objective is unchanged by any bi-measurable
  remeasurement of the state, and what that does \emph{not} give a fitted model.
\item[G. What a coordinate buys.] Where the algebra objective spends a coordinate, asked
  twice. Against a \emph{rank} budget on a product system, in closed form: products
  and powers are free, a new component is what costs. Against a
  \emph{reconstruction} objective: that one allocates by variance, this one by
  predictability.
\item[H. The predictive metrics and the intrinsic dimension.] The construction
  \Cref{thm:diffusion-equivalence} rests on --- the forward and backward metrics,
  the joint embedding, \Cref{def:intrinsic-dimension}, and the bound $d^\ast\le d$
  for smooth systems (\Cref{prop:dstar-bound}).
\item[I. Proof of \Cref{thm:diffusion-equivalence}.] The budget theorem. The
  \emph{definition} it depends on is \Cref{def:intrinsic-dimension}, in H.
\item[J. Experimental details and reproducibility.] Compute, seeds, metrics, and
  the hyperparameter ranges searched.
\end{description}

\section{The Projection Lemma}
\label{sec:proofs}

Throughout, write $H := L^2(\mcX \times \mcX, \bar{\mu})$ for the joint space, and view
every operator in \Cref{lem:induced-projection} as acting on $H$. We identify
$L^2(\mcX,\mu)$ with $V_X$ via the isometry $g \mapsto g \circ \mathrm{pr}$,
$(g\circ\mathrm{pr})(x,x') = g(x)$ --- an isometry because the first marginal of
$\bar\mu$ is $\mu$ --- and likewise $L^2(\mcX,\mu')$ with $V_{X'}$ via $f \mapsto f(x')$,
the second marginal being $\mu'$. Under these identifications $T$ is the lift
$f(x') \mapsto (Tf)(x)$, a map $V_{X'} \to V_X$. For a sub-$\sigma$-algebra $\mcG$ write
$V_\mcG$ for its closed subspace of measurable functions in $H$ and $P_{V_\mcG}$ for the
orthogonal projection onto it. The present algebra $\Vphi\subseteq V_X$ is then the
functions of $\phi(X)$, and the future algebra $\Vphip\subseteq V_{X'}$ the functions of
$\phi(X')$.

\begin{lemma}[Induced projection]
\label{lem:induced-projection}
As operators on $H$, the transfer operator factors into projections
$T=P_{V_X}P_{V_{X'}}$, and the operator of the coupling induced by any measurable $\phi$
is the two-sided compression of $T$ --- the present projected onto $\Vphi$, the future
onto $\Vphip$,
\begin{equation}
\label{eq:operator_projection_x_y}
  T_{Y,Y'}\;=\;P_{\Vphi}\,T\,P_{\Vphip} .
\end{equation}
\end{lemma}

The proof rests on two elementary facts.

\begin{description}
\item[(F1) Conditional expectation is orthogonal projection.] For a sub-$\sigma$-algebra
$\mcG$ of the product $\sigma$-algebra, $P_{V_{\mcG}} h = \Exp{h \cond \mcG}$ for every
$h \in H$: the conditional expectation is the $\mcG$-measurable function minimizing
$\norm{h - g}_{L^2}$ over $\mcG$-measurable $g$, which is exactly the orthogonal
projection of $h$ onto the closed subspace $V_{\mcG}$. In particular $P_{V_X} h$ depends
only on $x$ and $P_{V_{X'}} h$ only on $x'$.
\item[(F2) Nested projections collapse.] If $A \subseteq B$ are closed subspaces with
orthogonal projections $P_A, P_B$, then $P_A P_B = P_B P_A = P_A$.
\end{description}

\begin{proof}[Proof of \Cref{lem:induced-projection}]
\emph{First claim, $T = P_{V_X} P_{V_{X'}}$.} Take $h \in H$. By (F1), $P_{V_{X'}} h$ is a
function of $x'$ alone, $f(x') := \Exp{h \cond X' = x'}$, so $P_{V_{X'}} h = f \in V_{X'}$,
and applying (F1) again,
\begin{equation*}
  P_{V_X} P_{V_{X'}} h = P_{V_X} f = \Exp{ f(X') \cond X = x } = (T f)(x).
\end{equation*}
Thus $P_{V_X} P_{V_{X'}}$ restricts to $T$ on $V_{X'}$ and annihilates $V_{X'}^{\perp}$
(where $P_{V_{X'}} = 0$), so the two operators coincide on $H$.

\emph{Second claim, \eqref{eq:operator_projection_x_y}.} The induced pair
$(Y,Y') = (\phi(X),\phi(X'))$ is itself a coupling, so the first claim applied to it gives
$T_{Y,Y'} = P_{\Vphi} P_{\Vphip}$. Because $\sigma(Y) \subseteq \sigma(X)$ we have
$\Vphi \subseteq V_X$, and symmetrically $\Vphip \subseteq V_{X'}$. Applying (F2) to each
nested pair --- $P_{\Vphi} P_{V_X} = P_{\Vphi}$ and $P_{V_{X'}} P_{\Vphip} = P_{\Vphip}$ ---
with the first claim,
\begin{align*}
  P_{\Vphi}\, T\, P_{\Vphip}
  &= P_{\Vphi} \, P_{V_X} P_{V_{X'}} \, P_{\Vphip}
   = (P_{\Vphi} P_{V_X}) (P_{V_{X'}} P_{\Vphip}) \\
  &= P_{\Vphi} P_{\Vphip}
   = T_{Y,Y'},
\end{align*}
which is \eqref{eq:operator_projection_x_y}.
\end{proof}

\begin{remark}
The argument uses only nesting, so the factorization $T=P_VP_{V'}$ underlying
\eqref{eq:operator_projection_x_y} holds for
\emph{any} pair of closed subspaces in place of $\Vphi \subseteq V_X$ and
$\Vphip \subseteq V_{X'}$ --- in particular for a reproducing-kernel subspace of the
algebra, which is what later licenses modeling the operator on an RKHS rather than on all
of $L^2$.
\end{remark}

\section{The Objective as a $\chi^2$ Divergence, and its Direct Estimation}
\label{sec:lsif-app}

This section carries the proofs and the attribution for
\Cref{subsec:objective-connections}: the criterion \emph{is} a
$\chi^2$ divergence, its optimal witness is the density ratio $r$ of
\eqref{eq:density-ratio}, and least-squares density-ratio estimation computes it
variationally, with no covariance inversion and no shared ridge.

The criterion is a functional of one object, the density ratio of the
coupling. Intrinsically it is the Radon--Nikodym derivative of the
joint law $P$ of $(X,X')$ against the product of its marginals,
\begin{equation}
\label{eq:density-ratio}
  r \;=\; \frac{dP}{d(\mu\otimes\mu')},
  \qquad r(x,\cdot)=\frac{dp(\cdot\cond x)}{d\mu'},
\end{equation}
which needs no base measure and lies in $L^2(\mu\otimes\mu')$ exactly when
$\chi^2(X;X')<\infty$. If $\mu,\mu'$ carry densities $p_X,p_{X'}$ with respect to a
common reference $dz$, then $r$ is represented by
$r(x,x')=p(x,x')/(p_X(x)p_{X'}(x'))=p(x'\cond x)/p_{X'}(x')$ --- the reference
cancels in the ratio --- and the integrals $\int(\cdot)\,dz$ below are taken against
that $dz$, with $\mu'(dz)=p_{X'}(z)\,dz$. Its diagonal (singular) expansion
\citep{lancaster1958structure,renyi1959measures} is
\begin{equation}
\label{eq:lancaster}
  r(x,x') \;=\; 1+\sum_{j\ge1}\sigma_j\,\phi_j(x)\,\psi_j(x'),
\end{equation}
with $\{\phi_j\}$ orthonormal in $L^2(\mu)$, $\{\psi_j\}$ orthonormal in
$L^2(\mu')$, and $(\sigma_j,\phi_j,\psi_j)$ the singular system of $T$
\eqref{eq:cond-operator}. Here $\sigma_1$ is the maximal correlation and $\sigma_0=1$
on constants.

\paragraph{The objective.} The criterion of \eqref{eq:objective} is
the total $\chi^2$ dependence between the embedded present and future,
$\norm{T_{\phi}}_{\HS}^2-1=\chi^2(\phi(X);\phi(X'))$, and at $\phi=\mathrm{id}$ it
is
\begin{equation}
\label{eq:chi2-dependence}
  \chi^2(X;X') \;=\; \iint (r-1)^2\, p_X(x)\,p_{X'}(x')\,dx\,dx' \;=\; \sum_{j\ge1}\sigma_j^2 .
\end{equation}
This scalar is twice the squared-loss mutual information, which is conventionally
defined as $\tfrac12\chi^2$, and is the Pearson mean-square contingency and the
total principal inertia
\citep{calmon2017principal}. It measures the total mass of
the singular spectrum.

\begin{definition}[$\chi^2$ divergence]
\label{def:chi2-div}
For probability measures $P\ll Q$ on a common space,
\begin{equation}
\label{eq:chi2-div}
  \chi^2(P\,\|\,Q)
  \;=\;\int\Big(\frac{dP}{dQ}-1\Big)^{2}dQ
  \;=\;\int\Big(\frac{dP}{dQ}\Big)^{2}dQ\;-\;1,
\end{equation}
and $\chi^2(P\,\|\,Q)=\infty$ if $P\not\ll Q$ or the integral diverges.
\end{definition}

It is the $f$-divergence of $f(t)=(t-1)^2$: nonnegative, zero exactly at $P=Q$, and
satisfying the data-processing inequality \citep{sugiyama2012density}.

\begin{lemma}[the objective is a $\chi^2$ mutual information]
\label{lem:chi2-mi}
Fix an embedding $\phi$, let $P_\phi$ be the joint law of the pair
$(\phi(X),\phi(X'))$ and $Q_\phi$ the product of its marginals. Then
\begin{equation}
\label{eq:objective-as-divergence}
  \norm{T_{\phi}}_{\HS}^2 \;=\; 1+\chi^2\big(P_\phi\,\|\,Q_\phi\big),
\end{equation}
the $\chi^2$ analogue of a mutual information: the divergence of the coupling from
the independent coupling with the same marginals.
\end{lemma}

\begin{proof}
When $P_\phi\ll Q_\phi$ the ratio is $r_\phi$, the object \eqref{eq:density-ratio}
of the embedded pair, and \eqref{eq:chi2-div} evaluates through the expansion
\eqref{eq:lancaster} to $\int r_\phi^2\,dQ_\phi-1=\sum_{j\ge1}\sigma_j(\phi)^2$,
which together with the constant mode is $\norm{T_{\phi}}_{\HS}^2$, as in
\eqref{eq:chi2-dependence}. When $P_\phi\not\ll Q_\phi$ the operator is not
Hilbert--Schmidt and both sides are infinite.
\end{proof}

\paragraph{The role of the ratio.} By \eqref{eq:chi2-div} the objective is the
squared $L^2(Q_\phi)$ distance of $r_\phi$ from the constant $1$ --- dependence
measured as distance from independence --- and by \eqref{eq:density-ratio} the rows
$r_\phi(y,\cdot)$ are exactly the predictive laws of the embedded system. So the
divergence and the predictor are two functionals of the one object $r_\phi$, and any
estimator that recovers $r_\phi$ recovers both at once.

\paragraph{Least-squares density-ratio estimation.} LSIF
\citep{kanamori2009least,sugiyama2012density} fits $r$ by least squares in $L^2(Q)$
without forming either density. A critic is a function $g(x,x')$ \emph{of the pair}
--- it stands in for $r$, which lives on the product space; its slice $g(x,\cdot)$
at the optimum is the predictive density weight of \eqref{eq:density-ratio}. For any
critic $g\in L^2(Q)$, using $\int rg\,dQ=\int g\,dP$ --- the defining property of
the ratio, which removes $r$ from the criterion,
\begin{equation}
\label{eq:lsif-objective}
  J(g)\;=\;2\!\int g\,dP-\int g^2\,dQ
  \;=\;\int r^2\,dQ\;-\;\norm{g-r}_{L^2(Q)}^2 ,
\end{equation}
so
\begin{equation}
\label{eq:lsif-sup}
  \sup_{g}\,J(g)\;=\;1+\chi^2(P\,\|\,Q),
  \qquad\text{attained at } g=r .
\end{equation}
This is \eqref{eq:objective-as-divergence} in variational form --- the $\chi^2$ case
of the convex-dual representation of $f$-divergences
\citep{nguyen2010estimating}; the same functional under the name squared-loss mutual
information is developed by \citet{suzuki2013sufficient}. Empirically the first
integral runs over the observed pairs and the second over cross-pairings of present
with future samples, which sample $Q$ for free. The two consequences used in the
body: first, there is no whitening inverse. The critic's capacity is one budget,
but it is spent by the least-squares fit of $r$ itself --- concentrated where the
dependence is --- rather than assigned by the marginal spectrum of the features as
the shared whitening ridge of the kernel realization assigns it. (Per-observable
regularization is a property of the downstream \emph{prediction} stage, where each
regressed measurable carries its own ridge; the critic does not provide it and is
not meant to.) Second, by \eqref{eq:lsif-objective} $J(g)\le 1+\chi^2$ for
\emph{every} $g$, so on a held-out split the estimate approaches the truth from
below and cannot be inflated by overfitting. It is not monotone in training --- an overfit critic
drifts from $r$ and $J$ falls --- so the stopping point is itself selected on the
held-out value. This is a property of the estimate, not of a critic class: the
bilinear (span) critic, whose training is monotone and stable on the decoupled
system, develops the same post-peak decline once the components are coupled. The
reported model is the held-out snapshot for every critic class alike.

\paragraph{The critic-class ablation.} Holding the functional, the data, the
$m=10$ embedding architecture, the optimizer, the schedule and the stopping rule
fixed, and changing only the critic class from the free MLP to the bilinear
$g=1+\phi(x)^\top\!P\phi(x')$, reproduces the span behaviour inside our own
machinery. Per-component bit $R^2$ under the neural readout, against a raw-state
ceiling of $0.26$ at $\kappa=0$ and $0.31$ at $\kappa=1$. At $\kappa=0$ the
bilinear critic gives $0.037$ (cools $0.93$), where the free critic gives
$0.271/0.242/0.261$ over three seeds. At $\kappa=1$ the bilinear critic gives
$0.119/0.111/0.119$, which is the floor benchmark $0.117$ to within seed noise,
where the free critic reaches $0.264$--$0.318$. At both couplings the bilinear
row sits at the floor to within a few hundredths, far below the ceiling the free
critic reaches, so the phenomenon does not depend on the span estimator being a
VAMPnet --- it follows from the critic class alone.

\paragraph{Estimating the dependence.} At a fixed $\phi$, $\chi^2(\phi(X);\phi(X'))$
is a divergence of the joint law of the $k$-dimensional pair
$(\phi(X),\phi(X'))$, an off-the-shelf density-ratio problem. Direct estimators
include least-squares importance fitting \citep{kanamori2009least}, squared-loss
mutual information and its dimension-reduction use
\citep{suzuki2012canonical,suzuki2013sufficient}, convex variational divergence
estimation \citep{nguyen2010estimating,nowozin2016fgan}, and kernel dependence
measures \citep{gretton2005measuring}, with \citet{sugiyama2012density} the
reference treatment. None of these diagonalizes $T$, so the objective is
spectrum-free to evaluate --- the singular system enters only the analysis of
\Cref{thm:diffusion-equivalence}.

\begin{remark}[span and algebra are critic classes]
\label{rem:critic-classes}
Restricting the critic decides which functional \eqref{eq:lsif-sup} computes. Write
$u=\phi(X)$ and $v=\phi(X')$ for the present and future features, whitened and
mean-free, and $\hat C_{01}=\Exp{uv^\top}$ for their cross-covariance across the lag.
A direct computation then gives
$\max_{A}J\big(1+u^\top\!Av\big)=1+\norm{\hat C_{01}}_{F}^2$ --- the span (VAMP-2)
score of the coordinates --- while the supremum over all of $L^2(Q_\phi)$ is
\eqref{eq:objective-as-divergence}, the algebra. A bilinear critic of rank $m$
estimates the top $m$ modes: mode counting re-enters exactly through the critic's
rank, and declining that restriction is the same choice as scoring the algebra in
place of the span. Each half is separately on record. The rank-restricted maximum
is the variational principle of \citet{wu2020variational}; the same bilinear
restriction of \eqref{eq:lsif-objective} recurs as the spectral contrastive loss
\citep{haochen2021provable}, in discrete plug-in form in correspondence analysis
\citep[\S4.6]{greenacre1984correspondence} and \citet{riba2020regularized}, and ---
closest to the present setting --- in \citet{turri2025self}, who identify its
optimal value with the VAMP-2 score for exactly the present--future ratio.
\citet{kostic2024neural} fit low-rank models of $r$ with this loss and state that
the full-model optimum is the $\chi^2$ divergence. Neither endpoint is new and we claim
neither. Reading them as one axis is a matter of vocabulary rather than of result: mode
counting enters exactly through the critic's rank, and declining that restriction
\emph{is} the choice of algebra over span. What the reading is for is that it makes the
value a quantity to \emph{maximize over the embedding}. These derivations take the ratio
on the full space, where the unrestricted value $1+\chi^2(X;X')$ is a constant of the
coupling. Ours is the pushforward ratio, whose value $1+\chi^2(\phi(X);\phi(X'))$ moves
with $\phi$ --- which is what makes it a criterion for choosing $\phi$ at all.
\end{remark}

\paragraph{What the estimate certifies.} By \eqref{eq:lsif-sup}, restricting $g$ to a class
only lowers the value, so for the fitted $\phi$ the held-out estimate is at most
$1+\chi^2(\phi(X);\phi(X'))=\norm{T_{\phi}}_{\HS}^2$, the criterion's value at that map.
The estimate therefore approaches its target from below and cannot be inflated by
overfitting the critic. What it certifies is constructive and one-sided: the reported value
is attained by a map we exhibit, so $m$ coordinates \emph{suffice} for it. Showing $m$
coordinates \emph{necessary} would need an upper bound over all $\phi$, which no lower
bound supplies and which we do not claim.

\paragraph{Convergence of the estimate with the budget.} \Cref{fig:frontier} is a
diagnostic for the estimator, not a measurement of any dimension. For each $m$ we train the
algebra objective on the composite system and record the held-out $\hat J$, on a grid
$m=1$--$20$ at each coupling. Two things are worth reading off it. First, the absolute
scale is set by the critic architecture --- values rise by roughly $10\%$ per doubling of
the critic width --- so only the shape of each curve is identified, not its height, and no
comparison against the closed-form ceiling of \Cref{prop:exact-value} should be drawn from
it. Second, the value achievable at budget $m$ is non-decreasing in $m$, since adjoining a
coordinate can only enlarge $\Vphi$, so any decline past a plateau is estimation error
rather than signal. Growth stabilizes on every curve, and past that point the curves move
by no more than the spread between optimizer restarts at a single $m$, so what remains
there is run-to-run scatter rather than a trend. Where growth stabilizes moves with the
coupling, later at stronger coupling. The plateau itself falls as $\kappa$ rises, the
criterion being a property of the coupling, and coupling spending part of it.

\begin{figure}[t]
\centering
\includegraphics[width=\figwidth]{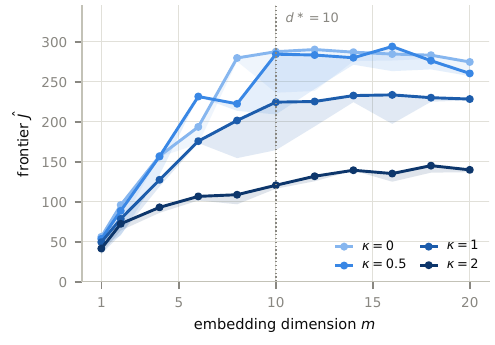}
\caption{Convergence of the objective with the coordinate budget, on the composite system
of \Cref{sec:experiments}. Held-out $\hat J$ against embedding dimension $m$ at each
coupling, seed spread shaded. The vertical scale is set by the critic architecture and is
not comparable across architectures, so the curves identify shape and not height. Because
the achievable value is non-decreasing in $m$, any decline past a plateau measures the
estimator's error rather than a loss of signal. Growth stabilizes on every curve, later at
stronger coupling. This is an
instrument diagnostic: the flattening reports that our estimator stops gaining, not that
the coupling has a dimension.}
\label{fig:frontier}
\end{figure}

\section{Additional Literature Notes}
\label{sec:additional-lit}

\paragraph{$\chi^2$ / dependence-maximization methods.}
Maximizing a $\chi^2$ dependence to select features is an established program.
Squared-loss mutual information drives sufficient dimension reduction
\citep{suzuki2013sufficient} over a projection that is linear, static, and aimed at
an external target --- it recovers the central subspace and carries no dimension
notion. Principal inertia components \citep{calmon2017principal} and Soft-HGR
\citep{huang2019efficient} maximize the same dependence over $k$ feature pairs and
recover the top-$k$ modes, and the spectral contrastive loss
\citep{haochen2021provable} characterizes its optimum as the top eigenfunctions ---
in our vocabulary, span readings, which count components by their modes.
Our objective instead credits every function of the coordinates, so once a
component's generators are held, its products and powers are free.

\citet{schmitt2023information} share a principal aim with our approach --- reducing a
system by how predictive its variables are of the future --- but measure predictability by a
mutual information against a bit-rate budget, through an information bottleneck.
That objective is considerably harder to compute: solved exactly it needs the full
conditional law $p(x_{t+\Delta t}\cond x_t)$, which they note is difficult to estimate
in practice, so their variational form replaces it with a prior on the code and a
contrastive estimate of the predictive information.

A similar objective, likewise built on mutual information, appears in
\citet{federici2024latent}, who maximize the mutual information between the embedded
present and the embedded future, and estimate it contrastively as well. Their
bottleneck term is not directly computable either, and is replaced by a bound requiring
an explicit conditional density model.

Both therefore inherit the ceiling of the contrastive estimate. $I_{\mathrm{NCE}}$ is
bounded by $\log B$ at batch size $B$ \citep{oord2018representation}, so representing
$I$ nats of dependence requires $B\ge e^{I}$, and since the information in
multicomponent systems typically grows roughly linearly in the number of components
(information is additive for independent components), the batch size required grows
exponentially in their number. Our objective carries no such bound: at any batch size
the held-out value is an unbiased estimate of the criterion at the fitted critic, and
the supremum over critics is $1+\chi^2$ exactly \eqref{eq:variational-body}.

\paragraph{Component identification.} The independent-Markov-decomposition line
\citep{hempel2021independent} and its deep successor \citep{mardt2022deep} build Markov
state models of biomolecular complexes. As discussed in \Cref{sec:intro}, the state
complexity of a product system grows with the number of components, and they approach
this by decomposing the system into weakly coupled subsystems.

\paragraph{Repeated eigendirections.} \citet{dsilva2018parsimonious} observe that a
diffusion-maps embedding returns coordinates parametrizing a direction that an earlier
coordinate already parametrizes --- \emph{repeated eigendirections}, their instance
being $\cos 2x$ against $\cos x$ --- which overstates the dimension of the data and
pushes a genuinely new direction below the repeats in the spectrum. These repeats are
exactly what the algebra objective credits for free: once $\cos x$ is held, every
function of it is available and \eqref{eq:objective} pays nothing further for
$\cos 2x$. In contrast to our approach, their solution is not to seek a better
embedding but to filter the one computed: each eigenvector is fitted from its
predecessors by a local linear regression, and is discarded if that fit succeeds.
Every kept direction must therefore first appear in the computed spectrum, and on the
system of \Cref{sec:experiments} the first fast component appears only past rank one
hundred --- a hundred eigenvectors and a hundred such tests to reach it. That count is
the product mode count, so it
grows exponentially in the number of slow components: the high-rank spectral
computation is paid in full and only the final dimension is reduced, where
\eqref{eq:objective} never forms a spectrum at all (\Cref{sec:lsif-app}).

\paragraph{Autoencoder methods.} Koopman autoencoders
\citep{lusch2018deep,takeishi2017learning} linearize the dynamics with a
finite-dimensional linear model. That requirement implies that a rank objective is the
optimal one, so here the rank restriction is structurally necessitated.
We require less: no model of the dynamics is fit at all, the embedding only generates the
algebra that \eqref{eq:objective} scores, and predictions are made from it afterwards by
regression.

Finally, we note that we use an autoencoder in \Cref{subsec:exp-warped}, for the purposes
of pre-training. This autoencoder does not model the \emph{dynamics}: it is trained to
reduce the dimension of present-only data, with no future and no lagged pairs, and is used
as a warm start for our objective. Adding the dynamic part, our objective itself, then
improves on it (\Cref{sec:aeonly-app}).

\section{Additional Experiments}

\subsection{The Warm Start Is Not the Advantage}
\label{sec:aeonly-app}

The few-label results of \Cref{subsec:exp-fewshot} use an algebra $\phi$ trained from
an autoencoder warm start (\Cref{subsec:exp-warped}), which raises the question of how
much of the advantage the unsupervised chart already supplies. We answer it by reading
out the chart itself, with no algebra training on top: the same architecture, the same
checkpoint the algebra runs start from, the same twenty paired label draws and the same
readout. At $300$ noisy labels on the warped observation, averaged over three seeds:

\begin{center}
\small
\begin{tabular}{lcc}
\hline
representation & cools & bits \\
\hline
direct regression              & $0.42$ & $0.030\pm0.038$ \\
autoencoder chart alone        & $0.49$ & $0.089\pm0.019$ \\
algebra $\phi$ from that chart & $0.78$ & $0.134\pm0.017$ \\
\hline
\end{tabular}
\end{center}

The chart alone carries real information about the fast components --- three times what
direct regression recovers --- so part of the gap to direct regression is due to
unsupervised pretraining rather than to the criterion. The remainder is due to the
criterion, and it is the larger part on the dominant components: training the algebra
objective from that chart adds $+0.29$ on the cools and $+0.045$ on the bits, winning
on all twenty paired draws in both cases (exact two-sided sign test,
$p=1.9\times10^{-6}$). The comparison is exact in the sense that the two rows differ
only in whether the objective was applied.

\begin{figure}[t]
\centering
\includegraphics[width=\figwidth]{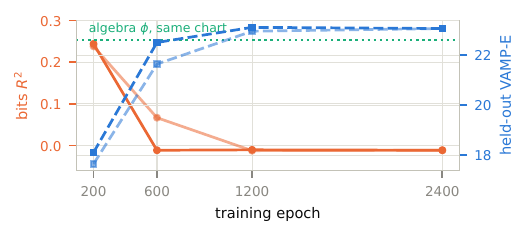}
\caption{Removal is the optimization. Initialized from the same unsupervised
chart, the span objective trains away readout-verified bit information (orange,
left axis) as its own held-out score improves (blue, right axis), while the
algebra objective from that chart retains it. Two restarts.}
\label{fig:warped}
\end{figure}

\subsection{Capacity and the Reachability Wall}
\label{sec:capacity-app}

An underfit span network does not build the interaction modes its own objective
calls for, and leaks bit information into the slots they vacate. Run at a $64^2$
trunk, the sweep of \Cref{fig:kfamily} shows the bits entering already at
$k\approx50$--$80$ --- a \emph{reachability} effect, not a sampling one, the span
estimator paying twice, once in rank and once in the capacity to compose roughly
one hundred interaction functions. That is why \Cref{fig:kfamily} is reported at
the largest trunks we train. This section records the per-mode evidence. At $\kappa=0$ the product eigenfunctions are
available in closed form, so we can measure, per mode, whether the trained
network can build them: a linear readout of each product mode from the trained
features shows the interaction frontier of a $64^2$ trunk frozen at low order
--- the cool-cool interaction $c_3{\times}c_3$ reads out at zero for every $k$
up to $120$, while at $256^2$--$512^2$ it is built at $60$--$80\%$ of its
ceiling. The sharpest control is a matched pair at essentially identical
singular value: the single-component mode $c_5{\times}1$ ($\sigma=0.838$) is
built by the $64^2$ trunk while the interaction $c_3{\times}c_2$
($\sigma=0.837$) is not --- the failure is composition, not spectral depth. The
same reading comes from the other end of the spectrum: the bits, at
$\sigma=0.55$ the lowest modes in play, are built well before interactions that
outrank them. The slots the unreachable interactions vacate go to the bits, and
the exchange is visible in both directions: down the $k$-axis at $64^2$ the
interaction frontier freezes while bits flood in between $k=50$ and $80$, and
up the width axis at fixed $k=80$ the interactions arrive while the bits drain
(bit $R^2$ $0.25\to-0.01$ from $64^2$ to $512^2$, which is $0.83\to-0.03$ of the
ceiling in the fraction-built units \Cref{fig:capacity} plots). Even at $512^2$ a small unbuildable
core remains (the deepest diagonal products), which is why admission in
\Cref{fig:kfamily} is gradual rather than sharp: the bits enter at the rank the
product count predicts, but reach only $0.19$ of their $0.26$ ceiling by
$k=120$, a fifth past the wall. The under-capacity and under-training artifacts
coincide numerically: at $\kappa=0$, $k=80$, the $512^2$ trunk reads $0.23$ at
$200$ epochs and decays to $-0.01$ by $2400$, the early value coinciding with
what a $64^2$ trunk reports at convergence.

\begin{figure}[t]
\centering
\includegraphics[width=\figwidth]{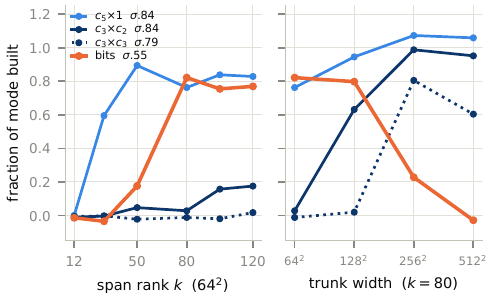}
\caption{The reachability wall, on both axes. Fraction of each mode built, as a
linear readout of that mode from the trained features ($R^2/\sigma^2$). Left,
against span rank at a $64^2$ trunk: the bits are built by $k=80$ even though
$\sigma=0.55$ makes them the \emph{lowest} modes plotted, while the interactions
above them stay unbuilt. Right, at fixed $k=80$ against trunk width: the
interactions arrive and the bits drain, the same slots changing hands.
$c_5{\times}1$ and $c_3{\times}c_2$ are a matched pair at $\sigma\approx0.84$,
built and unbuilt respectively, so what fails is composition rather than spectral
depth.}
\label{fig:capacity}
\end{figure}

\section{The Predictive Metric}
\label{sec:predictive-metric-app}

The objective of \Cref{sec:lsif-app} measures the \emph{size} of the singular
spectrum. This section concerns its \emph{geometry}: the predictive metric $\dS$
of \eqref{eq:diffusion-metric} and its embedding $\Jf$. Like the objective, $\dS$
is a functional of the density ratio $r$ of \eqref{eq:density-ratio} --- but
unlike the objective it \emph{is} the singular system.

The intrinsic dimension is built from $\dS$ but is not its dimension. By
\Cref{def:intrinsic-dimension}, $d^\ast$ is the box dimension of the \emph{joint}
embedding $J=(\Jf,\Jb)$ under the joint metric $\dJ$, which separates two states
whenever they differ in what they predict \emph{or} in what they postdict. The
forward half alone carries $\dS$ and has a dimension of its own, in general a
smaller one. \Cref{sec:dimension-app} assembles $\dJ$ from $\dS$ and its backward
counterpart and defines $d^\ast$ there.

The diffusion metric of
\eqref{eq:diffusion-metric} is the $\chi^2$ distance between one-step predictive
laws,
\begin{equation}
\label{eq:chi2-distance}
  \dS(x,y)^2
  = \int \frac{\big(p(z\cond x)-p(z\cond y)\big)^2}{p_{X'}(z)}\,dz
  = \norm{r(x,\cdot)-r(y,\cdot)}_{L^2(\mu')}^2
  = \sum_{j\ge1}\sigma_j^2\big(\phi_j(x)-\phi_j(y)\big)^2 .
\end{equation}
It admits three readings: the $\chi^2$ distance between the predictive laws
$p(\cdot\cond x)$ and $p(\cdot\cond y)$, the $L^2(\mu')$ distance between the rows
$r(x,\cdot)$ of the density ratio, and the weighted Euclidean distance in the
singular coordinates $\Jf(x)=(\sigma_j\phi_j(x))_j$. It is the
correspondence-analysis $\chi^2$ distance between the row profiles of the kernel
\citep{greenacre1984correspondence} and the diffusion
distance \citep{coifman2006diffusion} of the symmetrized round-trip operator
$TT^\ast$. Unlike the asymmetric $\chi^2(P\,\|\,Q)$, it is symmetric with the fixed
reference $\mu'$, and it vanishes exactly when $p(\cdot\cond x)=p(\cdot\cond y)$,
that is, when $x$ and $y$ are the same predictive state.

\paragraph{The kernel of $TT^\ast$.} Let $T^\ast:L^2(\mu)\to L^2(\mu')$ be the
adjoint of $T$, the backward conditional expectation
$(T^\ast g)(x')=\Exp{g(X)\cond X'=x'}$. The round-trip operator
$TT^\ast:L^2(\mu)\to L^2(\mu)$ --- step forward $x\mapsto x'\sim p(\cdot\cond x)$,
then back $x'\mapsto x''$ by the posterior $p(\cdot\cond x')$ --- is self-adjoint,
positive, and Markov ($TT^\ast\mathbf 1=\mathbf 1$), hence a reversible Markov
operator with stationary law $\mu$ and $TT^\ast\phi_j=\sigma_j^2\phi_j$. Its
integral kernel with respect to $\mu$, defined by
$(TT^\ast f)(x)=\int K(x,y)f(y)\,d\mu(y)$, is
\begin{equation}
\label{eq:tt-kernel}
  K(x,y)
  = \big\langle r(x,\cdot),\,r(y,\cdot)\big\rangle_{L^2(\mu')}
  = \int \frac{p(z\cond x)\,p(z\cond y)}{p_{X'}(z)}\,dz
  = 1+\sum_{j\ge1}\sigma_j^2\,\phi_j(x)\,\phi_j(y).
\end{equation}
This is the $\chi^2$ affinity between the two predictive laws, and its diagonal
$K(x,x)=1+\chi^2\big(p(\cdot\cond x)\,\|\,\mu'\big)$ measures how far the future of
$x$ departs from the average future $\mu'$. The predictive metric is exactly the
distance induced by this kernel,
\begin{equation}
\label{eq:kernel-distance}
  \dS(x,y)^2 = K(x,x)-2K(x,y)+K(y,y)
  = \norm{r(x,\cdot)-r(y,\cdot)}_{L^2(\mu')}^2 ,
\end{equation}
so $\dS$ is the diffusion distance \citep{coifman2006diffusion} of the symmetrized
operator $TT^\ast$ with $K$ as its kernel, at diffusion time $t=\tfrac12$ in the
convention $D_t^2=\sum_j\lambda_j^{2t}(\Delta\phi_j)^2$ with $\lambda_j=\sigma_j^2$.
Passing to $TT^\ast$ is what
lets the embedding machinery for self-adjoint operators apply to the non-reversible
coupling $T$. We avoid the term \emph{spectral dimension} for $d^\ast$: in the
Dirichlet-form literature that name is reserved for the exponent $d_s=2\alpha/\beta$
read off the on-diagonal heat-kernel decay $p_t(x,x)\asymp t^{-\alpha/\beta}$, a
short-time quantity that a metric taken at a \emph{fixed} lag cannot see.

\begin{remark}[The canonical RKHS of a coupling]
\label{rem:canonical-rkhs}
Equations~\eqref{eq:tt-kernel} and~\eqref{eq:kernel-distance} exhibit $\dS$ as a reproducing-kernel
distance. The feature map is $x\mapsto r(x,\cdot)=dp(\cdot\cond x)/d\mu'\in
L^2(\mu')$, its reproducing kernel is $K$, and
$\dS(x,y)=\norm{K(x,\cdot)-K(y,\cdot)}_{\mcH_K}$. The space is intrinsic to the
coupling: $\mcH_K=\operatorname{range}(T)$, the predictable observables
$f(x)=\Exp{g(X')\cond X=x}$, and $\norm{f}_{\mcH_K}$ is the $L^2(\mu')$ norm of the
minimal $g$ that predicts $f$, so less predictable directions (smaller $\sigma_j$)
carry more norm. In particular every predictable observable is $\dS$-Lipschitz: for
$f=Tg$, $|f(x)-f(y)|=\big|\langle g,\,r(x,\cdot)-r(y,\cdot)\rangle_{L^2(\mu')}\big|\le
\norm{g}_{L^2(\mu')}\,\dS(x,y)$, and the supremum over $\norm{g}_{L^2(\mu')}\le1$ gives
the dual form $\dS(x,y)=\sup\{\,|Tg(x)-Tg(y)|:\norm{g}_{L^2(\mu')}\le1\,\}$, an
integral-probability metric \citep{gretton2005measuring}. Thus $\dS$ is the smallest
metric in which $T$ smooths --- the coordinates $\sigma_j\phi_j$ are each $1$-Lipschitz
and $\Jf$ is an isometry onto $M_T$ \citep{coifman2006diffusion} --- and this
operator-analytic half of \Cref{thm:diffusion-equivalence} is unconditional, leaving all
of the reduction's difficulty in the geometry of $M_T$ rather than the action of $T$.
Unlike a method that fixes a kernel on the observation space,
nothing is chosen here: the feature $r(x,\cdot)$ is a Radon--Nikodym derivative,
hence invariant under a common bi-measurable reparametrization $x\mapsto h(x)$
(\Cref{lem:invariance}), and $K$ with it. Any characteristic kernel recovers the same
canonical system $(\sigma_j,\phi_j,\psi_j)$ after whitening
\citep{fukumizu2007statistical}, so whitening is exactly the passage from a chosen
kernel to $K$, the whitened linear instance being VAMP \citep{wu2020variational}. The
kernel $K$ itself is classical (the kernel of $TT^\ast$), and it is finite precisely
in the Hilbert--Schmidt regime,
$K(x,x)=1+\chi^2\big(p(\cdot\cond x)\,\|\,\mu'\big)<\infty$, degenerating in the
deterministic limit where $r(x,\cdot)$ leaves $L^2(\mu')$.
\end{remark}

\paragraph{Estimating the predictive metric.} The metric $\dS$ requires the
predictive laws $p(\cdot\cond x)$ or, equivalently, the leading singular functions
$\sigma_j\phi_j$. These are the estimands of nonlinear canonical analysis:
alternating conditional expectations \citep{breiman1985estimating}, functional
canonical variates \citep{buja1990remarks}, nonparametric canonical correlation
\citep{michaeli2016nonparametric}, kernel canonical correlation
\citep{bach2002kernel,fukumizu2007statistical}, and conditional mean embeddings or
kernel transfer operators \citep{song2009hilbert,klus2020eigendecompositions}. The
whitened linear case is the VAMP score \citep{wu2020variational}, and closed forms
are available when the singular functions are orthogonal polynomials
\citep{makur2016polynomial}. The reversible special case recovers the classical
diffusion map \citep{coifman2006diffusion}.

\paragraph{The kernel-imported metric.} The nearest prior construction of a reduced
predictive embedding, the reproducing-kernel $\epsilon$-machine
\citep{shalizi2001computational,brodu2020discovering}, metrizes predictive laws by a
maximum-mean-discrepancy \citep{gretton2005measuring} induced by a chosen kernel
rather than by $\dS$. That distance depends on the kernel and is not invariant under
a bi-measurable reparametrization of the observations, whereas $\dS$ is the distance
of the canonical kernel $K$ (\Cref{rem:canonical-rkhs}) --- an arbitrary-kernel proxy
in place of the coupling's own. Because box dimension is a metric --- not a
topological --- invariant, the two give different dimensions, and the intrinsic
dimension of \Cref{def:intrinsic-dimension} is the box dimension under the canonical
$\dS$.

\paragraph{The feature class carries a topology, and it can bias $d^\ast$.} The lower
bound survives only while every $\phi$ in the maximization is admissible, that is,
continuous in $\dJ$ (\Cref{def:intrinsic-dimension}). A network on
$\mcX\subseteq\RR^d$ instead enforces continuity in the \emph{native} metric, and the
two classes are in general incomparable. If the singular functions $\phi_j$ are
native-smooth, a native-continuous $\phi$ can still separate $\dS$-equivalent states
--- it is then $\dS$-discontinuous, and the value it attains is not one a
$\dS$-continuous map could certify. If instead the $\phi_j$ are
native-rough or discontinuous, as under a non-smooth observation of the state, the
admissible optimum is native-discontinuous and lies \emph{outside} the network class,
which approximates it by several smooth coordinates and \emph{over}-counts $d^\ast$.
Enlarging the class past admissibility therefore deflates the dimension below its
invariant value: the operative constraint is fidelity to $\dS$, not network capacity.

\paragraph{A shared limitation the algebra softens.} This native-continuity prior is
common to all neural transfer-operator estimators --- VAMPnets
\citep{mardt2018vampnets}, deep canonical correlation \citep{andrew2013deep}, Koopman
autoencoders \citep{lusch2018deep} --- none of which can represent a
native-discontinuous singular function, since the invariance resides in the operator
(the $\sigma_j$ and the $\phi_j$ as $L^2$ objects) and any network reaches it only
through the native topology. Relative to these span methods the algebra is the less
exposed: a mode method must carry every retained mode as a network output, including
the roughest high harmonics, whereas here the network represents only the generators,
and their products --- the harmonics --- are scored through the algebra rather than
represented. The network thus carries the fewest and smoothest functions, so the
native-smoothness bottleneck binds later than for span methods, and only once the
generators themselves become discontinuous.

\section{Invariance}
\label{sec:invariance-app}

Let $h:\mcX\to\tilde\mcX$ be measurable and injective with measurable inverse on its image,
and form the remeasured coupling $\tilde X=h(X)$, $\tilde X'=h(X')$, with $\tilde\mu=h_*\mu$,
induced $\tilde\mu'=h_*\mu'$, and operator $\tilde T=T_{\tilde X,\tilde X'}$. Pulling
functions back along $h$, $U_h:L^2(\tilde\mu)\to L^2(\mu)$, $(U_h u)(x)=u(h(x))$, is an
isometry because $\tilde\mu=h_*\mu$, and it is \emph{onto} because $h^{-1}$ is measurable on
$h(\mcX)$: any $v\in L^2(\mu)$ is $v=U_h(v\circ h^{-1})$. Hence $U_h$ is unitary, and
$U_{h'}$ is its analogue $L^2(\tilde\mu')\to L^2(\mu')$. Surjectivity of $h$ plays no role,
since $\tilde\mcX\setminus h(\mcX)$ is $\tilde\mu$-null. When $\mcX$ and $\tilde\mcX$ are
standard Borel the measurability of $h^{-1}$ is automatic, injectivity alone sufficing by
the Lusin--Suslin theorem \citep{kechris1995classical}. No surjectivity is required, so
observations that embed the state in a \emph{larger} space, as in the warped and
high-dimensional observations of \Cref{sec:experiments}, are covered. Injectivity is
exactly what makes $U_h$ unitary rather than merely isometric --- for a general measurable
$h$ one gets $\norm{\tilde T_{\tilde Y,\tilde Y'}}_{\HS}\le\norm{T_{Y,Y'}}_{\HS}$, the
compression inequality of \Cref{lem:induced-projection}.

\begin{lemma}[Invariance]
\label{lem:invariance}
For any $\tilde\phi:\tilde\mcX\to\RR^m$ put $\phi=\tilde\phi\circ h$, $Y=\phi(X)$, and
$\tilde Y=\tilde\phi(\tilde X)$. The induced operators are unitarily conjugate,
\begin{equation}
\label{eq:invariance}
  \tilde T_{\tilde Y,\tilde Y'} \;=\; U_h^{-1}\,T_{Y,Y'}\,U_{h'},
\end{equation}
so $\norm{\tilde T_{\tilde Y,\tilde Y'}}_{\HS}=\norm{T_{Y,Y'}}_{\HS}$: the criterion gives
$\tilde\phi$ on the relabeled coupling the same value it gives $\tilde\phi\circ h$ on the
original. The predictive metric is fixed by the singular system, so the same conjugation
carries $\dJ$ along.
\end{lemma}

\begin{proof}[Proof of \Cref{lem:invariance}]
Conditioning on $\tilde X=\tilde x$ is conditioning on $X=h^{-1}(\tilde x)$, so for
$u\in L^2(\tilde\mu')$,
\begin{equation*}
  (\tilde T u)(\tilde x)
  =\Exp{u(h(X'))\cond X=h^{-1}(\tilde x)}
  =(U_h^{-1}\,T\,U_{h'}\,u)(\tilde x),
\end{equation*}
which is \eqref{eq:invariance} for the full algebras. For the restriction, $g$ is
$\sigma(\tilde\phi)$-measurable iff $g=\rho\circ\tilde\phi$ for some $\rho$, and then
$U_h g=\rho\circ\phi$ is $\sigma(\phi)$-measurable. Conversely every $\sigma(\phi)$-measurable
$\rho\circ\phi$ equals $U_h(\rho\circ\tilde\phi)$, so $U_h V_{\tilde\phi}=\Vphi$ --- this
step needs no injectivity, only $\phi=\tilde\phi\circ h$. Since $U_h$ is unitary, by fact
\textbf{(F1)} a unitary carries the projection onto a
subspace to the projection onto its image, $P_{V_{\tilde\phi}}=U_h^{-1}P_{\Vphi}U_h$ and
$P_{V'_{\tilde\phi}}=U_{h'}^{-1}P_{\Vphip}U_{h'}$. Substituting into
$\tilde T_{\tilde Y,\tilde Y'}=P_{V_{\tilde\phi}}\tilde T P_{V'_{\tilde\phi}}$,
\begin{equation*}
  \tilde T_{\tilde Y,\tilde Y'}
  =U_h^{-1}P_{\Vphi}U_h\cdot U_h^{-1}T U_{h'}\cdot U_{h'}^{-1}P_{\Vphip}U_{h'}
  =U_h^{-1}\,P_{\Vphi}T P_{\Vphip}\,U_{h'},
\end{equation*}
which is \eqref{eq:invariance}. Unitary conjugation preserves the Hilbert--Schmidt norm.
The predictive metric is a function of the operator's singular system, which the same
equivalence carries along, so $\dJ$-continuity is preserved. Finally
$\tilde\phi\mapsto\tilde\phi\circ h$ is a bijection between feature maps on $\tilde\mcX$ and
on $\mcX$, with inverse $\phi\mapsto\phi\circ h^{-1}$ --- here injectivity is used --- so the
two problems have the same achievable values at every budget.
\end{proof}

\begin{remark}[Invariance of the fitted pipeline]
\label{rem:homeomorphism-invariance}
\Cref{lem:invariance} is a statement about the criterion, not about any estimator. A fitted
$\phi$ is drawn from a class defined by regularity in the observed coordinates, and that
class is preserved by $h$ only when $h$ preserves the native topology. Asking $h$ to be a
\emph{homeomorphism} --- bi-continuous, not smooth --- suffices, and as a bi-measurable map
it still leaves $\dJ$ and $d^\ast$ fixed, so the estimator is equivariant and not only the
target. This is also the class on which a kernel fixed on the raw observations fails,
ambient smoothness being a metric and not a topological invariant.
\Cref{subsec:exp-warped} tests it.
\end{remark}

\section{What a Coordinate Buys}
\label{sec:budget-app}

Recall that the \emph{algebra objective} scores a map $\phi$ by the size of the operator of the
coupling $\phi$ induces, $\norm{T_{\phi}}_{\HS}^2=1+\chi^2(\phi(X);\phi(X'))$, which
is \eqref{eq:objective}. This section evaluates that score exactly in two settings
where it can be set against what a different criterion would choose from the same
budget. The first asks where a \emph{rank} budget goes on a system assembled from
independent components. The second asks the same of a \emph{reconstruction}
objective.

Both pose one question --- given one more coordinate, what is it spent on --- and
the two rivals answer it in the same way as each other. A rank method ranks the
candidates by singular value and buys the largest. A reconstruction objective ranks
directions by variance and keeps the largest. Each scores a candidate by a magnitude
the candidate carries \emph{on its own}, with no reference to what the coordinates
already held can express. The algebra objective \eqref{eq:objective} scores it instead by
what it \emph{adds}: a coordinate is worth something exactly insofar as it enlarges
the algebra $\sigma(\phi)$ in a direction carrying present--future dependence. So a
mode that is already a product of coordinates held is worth nothing however large
its singular value, and a direction that decorrelates within one step is worth
nothing however large its variance. The two computations below are those two
sentences, made exact.

\subsection*{Products are free, new components are not}

Assume $X=(X_1,\dots,X_m)$ has independent coordinates and the transition
factorizes, $p(x'\cond x)=\prod_i p_i(x_i'\cond x_i)$, so the components evolve
without interacting. Let $T_i$ be the operator of component $i$, with singular
values $1=\sigma_{i,0}>\sigma_{i,1}\ge\cdots$ and left singular functions
$\phi_{i,n}$, and write
\begin{equation}
\label{eq:factor-mass}
  g_i \;:=\; \norm{T_i}_{\HS}^2 \;=\; \sum_{n\ge 0}\sigma_{i,n}^2
\end{equation}
for its total spectral mass, every order included. Call a component
\emph{non-degenerate} if its leading singular function $\phi_{i,1}$ generates
$\sigma(X_i)$, which holds whenever $\phi_{i,1}$ is injective. Two elementary
components qualify. A Gaussian one --- $X_i\sim\mcN(0,1)$ with $(X_i,X_i')$ jointly
Gaussian of correlation $\rho_i$ --- has Hermite singular functions and
$\sigma_{i,n}=\rho_i^{\,n}$, so $g_i=1/(1-\rho_i^2)$. A two-state one ---
$X_i\in\{\pm1\}$ uniform, flipping with probability $(1-\lambda_i)/2$ --- has the
single nontrivial mode $\phi_{i,1}=x_i$ at $\sigma_{i,1}=\lambda_i$, so
$g_i=1+\lambda_i^2$. The second is the idealization of the bits of
\Cref{sec:experiments}.

\begin{proposition}
\label{prop:exact-value}
Suppose every component is non-degenerate, and let $\phi$ select a subset
$I\subseteq[m]$ of them through their generating coordinates $(x_i)_{i\in I}$. Then
\begin{equation}
\label{eq:exact-value}
  \norm{T_{\phi}}_{\HS}^2 \;=\; \prod_{i\in I} g_i ,
  \qquad\text{i.e.}\qquad
  \chi^2\big(\phi(X);\phi(X')\big) \;=\; \prod_{i\in I} g_i \;-\; 1 .
\end{equation}
Consequently: (1) selecting every component gives the unconditional maximum
$\norm{T}_{\HS}^2=\prod_i g_i$; (2) \emph{products are free} --- adjoining a product
$\phi_{i,1}\phi_{j,1}$ of two already-selected components leaves the value
unchanged, so a coordinate spent on it buys nothing; (3) when $\sigma_{\ell,1}>0$,
spending that same coordinate on an unselected component $\ell\notin I$ raises the
value by $\big(\prod_{i\in I}g_i\big)(g_\ell-1)>0$.
\end{proposition}

\begin{proof}
The algebra objective sees $\phi$ only through the $\sigma$-algebra it generates, and
non-degeneracy gives $\sigma(\phi)=\sigma\big((X_i)_{i\in I}\big)$, so we may take
$\phi(x)=(x_i)_{i\in I}$. The induced pair
$(Y,Y')=\big((X_i)_{i\in I},(X_i')_{i\in I}\big)$ is then a product coupling in its
own right: its joint law is $\bigotimes_{i\in I}\mathrm{law}(X_i,X_i')$ and its two
marginals are the corresponding products, so its density ratio against the product
of its marginals factorizes, $r_Y(y,y')=\prod_{i\in I}r_i(y_i,y_i')$ with
$r_i=dp_i(\cdot\cond x_i)/d\mu_i'$. By \eqref{eq:hs-chi2-prelim} applied to $(Y,Y')$
and Fubini,
\begin{equation*}
  \norm{T_{\phi}}_{\HS}^2=\int r_Y^2\,d(\mu_Y\otimes\mu_Y')
  =\prod_{i\in I}\int r_i^2\,d(\mu_i\otimes\mu_i')
  =\prod_{i\in I}\norm{T_i}_{\HS}^2=\prod_{i\in I} g_i .
\end{equation*}
Claim (1) is $I=[m]$. For (2), $\phi_{i,1}\phi_{j,1}$ is a function of $(x_i,x_j)$,
so adjoining it leaves $\sigma(\phi)$, hence the value, unchanged. For (3),
adjoining component $\ell$ replaces $I$ by $I\cup\{\ell\}$ and multiplies the value
by $g_\ell$.
\end{proof}

Only the law of the induced pair enters, so no compression of $T$ and no singular
system of the joint operator is needed --- the product structure of $(Y,Y')$ does
all the work. The singular values appear only to name $g_i$ in
\eqref{eq:factor-mass} and to identify which coordinate a rank method would pick.

\paragraph{What the two claims mean.} Claims (2) and (3) are one coordinate spent
two ways. Suppose $\phi=(x_1,x_2)$ and one further coordinate is available. A
candidate for it is the product $\phi_{1,1}\phi_{2,1}$ --- not an arbitrary choice,
but a genuine singular function of the joint operator, of singular value
$\lambda_1\lambda_2$, and when $\lambda_1\lambda_2>\lambda_3$ it is the
\emph{third-largest mode of the entire system}, so a rank-$3$ method returns exactly
$x_1,x_2,x_1x_2$. Yet that product is already a function of $(x_1,x_2)$. The
generated algebra does not grow, $I$ is still $\{1,2\}$, and the value is still
$g_1g_2$: the coordinate buys nothing whatsoever. Spent on $x_3$ instead it buys
$g_1g_2(g_3-1)>0$, strictly positive as soon as the third component has any
nontrivial mode. That is the double meaning of \emph{products are free}. The
algebra objective credits every product of the coordinates already held, so a coordinate
need never be spent on one --- and therefore a coordinate spent on one is spent on
something already owned.

\paragraph{The smallest instance, and the one we run.} Take three two-state
components with $\lambda_1\lambda_2>\lambda_3$. The rank-$3$ optimum is
$\{x_1,x_2,x_1x_2\}$, whose algebra is $\sigma(X_1,X_2)$ and therefore contains no
non-constant function of $X_3$ at all: the third component is not modeled coarsely,
it is absent. The algebra objective selects $\{x_1,x_2,x_3\}$ and exceeds the rank-$3$ value
by $(1+\lambda_1^2)(1+\lambda_2^2)\lambda_3^2$. When $\lambda_3>\lambda_1\lambda_2$
the two selections coincide. So the mode ordering decides whether a rank budget goes
wrong, and never affects what \eqref{eq:objective} prefers. This is linear masking
at $m=3$, exactly. The system of \Cref{sec:experiments} is the same phenomenon with
eight components, a budget near one hundred, and measured rather than closed-form
quantities --- a fuller demonstration, though not a sharper one than this.

\begin{remark}
A feature violating non-degeneracy --- for instance $\phi=x_i^2$, which discards the
sign of $x_i$ and so captures only the even-order modes of component $i$ --- can
only lower the value, never exceed the all-component maximum of claim (1).
\end{remark}

\subsection*{Autoencoder Reconstruction vs Dynamic Objective}

An autoencoder model of the dynamics seeks an encoder $E:\mcX\to\RR^k$, a decoder
$D:\RR^k\to\mcX$, and a latent map $M$ --- linear in the Koopman autoencoders of
\citet{lusch2018deep} --- with
\begin{equation}
\label{eq:autoencoder}
  D\circ E(x)\approx x, \qquad D\circ M\circ E(x)\approx x' .
\end{equation}
The reconstruction requirement makes \eqref{eq:autoencoder} allocate the budget $k$
by \emph{variance} rather than by dynamics: minimizing reconstruction error rewards
$E$ for keeping the directions of $x$ that are most costly to regenerate, the
high-variance ones, whether or not they carry any dynamical information. Write
$x=x_1\oplus x_2$ with $x_1$ high variance but decorrelating in one step --- a
per-step nuisance --- and the predictable dynamics in a lower-variance $x_2$.
Reconstruction spends coordinates on $x_1$, since dropping it dominates the error,
whereas \eqref{eq:objective} spends none, since $x_1$ contributes no dependence
between $X$ and $X'$. In the factorized setting of \Cref{prop:exact-value} this is
exact: a per-step nuisance factor has $g_i=1$ and zero selection gain in claim (3),
at any variance. (A \emph{static} attribute --- $x_1$ frozen, $x_1'=x_1$ --- is the
opposite case: it is a $\sigma=1$ mode that \eqref{eq:objective} does retain, being
perfectly coupled across time.) Reconstruction methods are in this sense a
nonlinear analogue of PCA --- a linear autoencoder \emph{is} PCA --- selecting
directions by variance, where \eqref{eq:objective} selects by predictability.

This is a property of the loss, not of the decoder, and dropping reconstruction
does not repair it. Keeping only $D\circ M\circ E(x)\approx x'$ --- modeling
$x'\cond x$ --- constrains the encoder only on the present support; under a
non-stationary coupling $\mu'\neq\mu$ the future occupies a different region, so
iterating the model applies $E$ off the support it was fit on. The role of the
latent also becomes unclear: a conditional generative model of $x'\cond x$ places
its latent on the transition noise --- the variation in $x'$ not determined by $x$
--- which is not a representation of the state at all. Recovering a representation
requires bottlenecking $x$, which is the one-sided predictive-sufficiency objective
once more, now committed to a likelihood we do not need.

Finally, even at its optimum the autoencoder recovers only the conditional mean:
trained under an $L_2$ loss, $D\circ M\circ E$ converges to $\Exp{X'\cond X=x}$, a
single predicted next state, and says nothing about the spread of $x'$ around it.
The operator $T$ instead encodes $\Exp{f(X')\cond X}$ for \emph{every} observable
$f$, and that family determines the full conditional law. The objective
\eqref{eq:objective} carries no decoder and never reconstructs $x$.

\section{The Predictive Metrics and the Intrinsic Dimension}
\label{sec:dimension-app}

This section carries the construction that \Cref{thm:diffusion-equivalence} rests on.
The singular system of $T$ puts two metrics on the state space, one forward and one
backward. Stacking their embeddings gives the joint embedding whose box dimension is
the intrinsic dimension $d^\ast$ of \Cref{def:intrinsic-dimension}. We then exhibit a
class of systems on which $d^\ast$ is finite and behaves as expected
(\Cref{prop:dstar-bound}), which is what makes the budget of
\Cref{thm:diffusion-equivalence} a finite one.

\paragraph{The predictive metrics.} The singular system puts a metric on the state
space. The \emph{forward} predictive distance between two states is the $\chi^2$
distance\footnote{The symbol $\chi^2$ carries two meanings. Earlier,
$\chi^2(X;X')$ in \eqref{eq:hs-chi2-prelim} is a \emph{dependence} between the pair
$(X,X')$, comparing their joint law with the product of the marginals. Here
$\chi^2\big(p(\cdot\cond x)\,\|\,p(\cdot\cond y)\big)$ is a \emph{divergence} between two
fixed distributions, the predictive laws of $x$ and $y$. It is taken against the reference
$\mu'$ rather than the second argument, which makes it symmetric --- the $\chi^2$ distance
of correspondence analysis \citep{greenacre1984correspondence}.} between the futures they
predict,
\begin{equation}
\label{eq:diffusion-metric}
\begin{aligned}
  \dS(x,y)^2
  &=\chi^2\!\big(p(\cdot\cond x)\,\big\|\,p(\cdot\cond y)\big)\\
  &=\norm{r(x,\cdot)-r(y,\cdot)}_{L^2(\mu')}^2\\
  &=\sum_{j\ge1}\sigma_j^2\big(\phi_j(x)-\phi_j(y)\big)^2,
\end{aligned}
\end{equation}
the weighted Euclidean distance in the left singular coordinates, vanishing exactly when
$x$ and $y$ predict the same future. Its time-reversed counterpart, the \emph{backward}
distance, is built identically from the right singular functions and the backward kernel
$q(\cdot\cond x)$ of $T^\ast$ (the previous state given the current),
\begin{equation}
\label{eq:backward-metric}
\begin{aligned}
  \dSb(x,y)^2
  &=\chi^2\!\big(q(\cdot\cond x)\,\big\|\,q(\cdot\cond y)\big)\\
  &=\sum_{j\ge1}\sigma_j^2\big(\psi_j(x)-\psi_j(y)\big)^2 .
\end{aligned}
\end{equation}
Both are coordinate-free: because $r$ is a Radon--Nikodym derivative, $\dS$ and $\dSb$ are
unchanged when the state is remeasured by any bi-measurable $h$.

\paragraph{The joint embedding and the intrinsic dimension.} The forward and backward maps
$\Jf(x)=(\sigma_j\phi_j(x))_{j\ge1}$ and $\Jb(x)=(\sigma_j\psi_j(x))_{j\ge1}$ embed the
state into $\ell^2$ as the predictive manifolds $M_T=\Jf(\mcX)$ and
$M_{T^\ast}=\Jb(\mcX)$, carrying $\dS$ and $\dSb$ \citep{coifman2006diffusion}. Stacking
them gives the \emph{joint} (bidirectional) embedding $J=(\Jf,\Jb)$, an isometry from
$(\mcX,\dJ)$ onto $J(\mcX)\subseteq\ell^2\oplus\ell^2$ under
\begin{equation}
\label{eq:joint-metric}
  \dJ(x,y)^2=\dS(x,y)^2+\dSb(x,y)^2 ,
\end{equation}
which separates two states whenever they differ in what they predict or in what they
postdict.

\begin{definition}[Intrinsic dimension]
\label{def:intrinsic-dimension}
The \emph{intrinsic dimension} of the coupling is the upper box-counting dimension of the
joint embedding under the joint metric,
\begin{equation}
\label{eq:intrinsic-dimension}
  d^\ast := \overline{\dim}_{\mathrm{box}}\big(J(\mcX),\dJ\big),
\end{equation}
the exponent governing the growth as $\eps\to0$ of the number of $\dJ$-balls of radius
$\eps$ needed to cover $J(\mcX)$ \citep{falconer2014fractal}. We take the upper version
throughout, which is the one the covering bounds below and the embedding theorem require,
and we assume $J$ is defined at every state, that is $K(x,x)$ and $K^\ast(x,x)$ are finite
pointwise rather than merely almost everywhere --- a standing regularity condition, needed
because $\mu$ and $\mu'$ may differ and no single null set is exceptional for both.
\end{definition}

The next proposition exhibits a class of systems with bounded intrinsic dimension:
smooth, nondegenerate dynamics on a compact $d$-dimensional state space.

\begin{proposition}[A class of systems with bounded intrinsic dimension]
\label{prop:dstar-bound}
Let $(\mcX,\rho)$ be a compact metric space with upper box dimension $d$, carrying a
finite reference measure $dz$, and suppose the transition law has a density
$p(z\cond x)$ with respect to $dz$ such that, for constants $0<c\le C<\infty$ and
$L<\infty$,
\begin{itemize}
\item[(i)] $c\le p(z\cond x)\le C$ for all $x,z$, and
\item[(ii)] $\lvert p(z\cond x)-p(z\cond y)\rvert\le L\,\rho(x,y)$ and
$\lvert p(x\cond w)-p(y\cond w)\rvert\le L\,\rho(x,y)$ for all $z,w$ and $x,y$.
\end{itemize}
Then for every initial law $\mu$ on $\mcX$ the coupling is Hilbert--Schmidt and
$d^\ast\le d$. In particular, \Cref{thm:diffusion-equivalence} guarantees that
$\lfloor2d\rfloor+1$ coordinates always suffice for such a system.
\end{proposition}

\begin{proof}
Write $\lvert \mcX\rvert=\int dz$ for the total reference mass. The future marginal has
density $p_{X'}(z)=\int p(z\cond w)\,d\mu(w)\in[c,C]$, so the density ratio
$r(x,z)=p(z\cond x)/p_{X'}(z)$ is bounded by $C/c$. Hence
$\norm{T}_{\HS}^2=\int r^2\,d(\mu\otimes\mu')\le(C/c)^2<\infty$, the coupling is
Hilbert--Schmidt, and $K(x,x)$ and $K^\ast(x,x)$ are finite at every state, so the
standing regularity condition of \Cref{def:intrinsic-dimension} holds.

For the forward metric, \eqref{eq:chi2-distance} with (i) and (ii) gives
\begin{equation*}
\dS(x,y)^2=\int\frac{\big(p(z\cond x)-p(z\cond y)\big)^2}{p_{X'}(z)}\,dz
\;\le\;\frac{L^2\lvert \mcX\rvert}{c}\,\rho(x,y)^2 .
\end{equation*}
For the backward metric, the reversed kernel at the future state $x$ has density
$g_x(w)=p(x\cond w)/p_{X'}(x)$ with respect to $\mu$, and
$\lvert p_{X'}(x)-p_{X'}(y)\rvert\le\int\lvert p(x\cond w)-p(y\cond w)\rvert\,d\mu(w)\le L\rho(x,y)$,
so for every $w$
\begin{equation*}
\lvert g_x(w)-g_y(w)\rvert
\;\le\;\frac{\lvert p(x\cond w)-p(y\cond w)\rvert}{p_{X'}(x)}
+p(y\cond w)\,\frac{\lvert p_{X'}(y)-p_{X'}(x)\rvert}{p_{X'}(x)\,p_{X'}(y)}
\;\le\;L\Big(\frac1c+\frac{C}{c^2}\Big)\rho(x,y),
\end{equation*}
whence $\dSb(x,y)^2=\int(g_x-g_y)^2\,d\mu\le L^2(1/c+C/c^2)^2\rho(x,y)^2$. The two
bounds make the joint embedding $J:(\mcX,\rho)\to(J(\mcX),\dJ)$ Lipschitz and
surjective, and Lipschitz maps do not increase upper box dimension
\citep{falconer2014fractal}, so
$d^\ast=\overline{\dim}_{\mathrm{box}}(J(\mcX),\dJ)\le
\overline{\dim}_{\mathrm{box}}(\mcX,\rho)=d$.
\end{proof}

The hypotheses hold, for example, for a nondegenerate diffusion observed at a fixed
lag on a compact manifold: the transition density is then smooth and strictly
positive, hence bounded and Lipschitz in each argument.

\section{Proof of \Cref{thm:diffusion-equivalence}}
\label{sec:diffusion-dim-app}

Every object this proof consumes is constructed in \Cref{sec:dimension-app}: the
forward and backward predictive metrics $\dS$ and $\dSb$, the joint embedding
$J=(\Jf,\Jb)$ and the joint metric $\dJ$ it carries as an isometry, and the
intrinsic dimension $d^\ast=\overline{\dim}_{\mathrm{box}}\big(J(\mcX),\dJ\big)$ of
\Cref{def:intrinsic-dimension}. This section takes them as given and proves the
budget statement about them.

\begin{proof}[Sketch]
Finite box dimension forces total boundedness, so $J(\mcX)$ is a precompact subset of the
Hilbert space $\ell^2\oplus\ell^2$ of box dimension $d^\ast$, and passing to the closure
gives a compact set of the same dimension. For any integer $m>2d^\ast$, embedology
\citep{hunt1999regularity} --- the infinite-dimensional form of \citet{sauer1991embedology}
--- makes a prevalent, hence at least one, bounded linear $\pi:\ell^2\oplus\ell^2\to\RR^m$
injective on $J(\mcX)$ with continuous inverse. This is an existence statement, not an
explicit map. The composite $\phi=\pi\circ J$ is $\dJ$-Lipschitz with constant
$\norm{\pi}$, because $J$ is a $\dJ$-isometry. Injectivity recovers $J$ from $\phi$, so
every $\sigma_j\phi_j$ and $\sigma_j\psi_j$ is a function of $\phi$: the left singular
functions become functions of $\phi(X)$ and the right ones of $\phi(X')$, whence
$\norm{T_{\phi}}_{\HS}=\norm{T}_{\HS}$. The full argument follows.
\end{proof}

Write $\mcE=\ell^2\oplus\ell^2$ and recall the joint embedding
$J(x)=\big((\sigma_j\phi_j(x))_{j},(\sigma_j\psi_j(x))_{j}\big)\in\mcE$, which is an isometry
of $(\mcX,\dJ)$ onto $J(\mcX)$: $\norm{J(x)-J(y)}_{\mcE}^2=\dS(x,y)^2+\dSb(x,y)^2=\dJ(x,y)^2$.
The singular system enters only as a choice of coordinates here. By \eqref{eq:diffusion-metric}
the same isometry is $x\mapsto\big(r(x,\cdot),r^\ast(\cdot,x)\big)\in
L^2(\mu')\oplus L^2(\mu)$, built from the conditional laws alone, and
$(\sigma_j\phi_j(x))_j$ is its coordinate vector in the singular basis. Everything below may
be read in either presentation.\footnote{The argument can also be closed without the singular
system. Since $\pi$ is injective on the compact $M$ with continuous inverse and $J$ is an
isometry, $\phi$ is injective on $\mcX/\{\dJ=0\}$ with Borel inverse, so
$r(x,x')=H\big(\phi(x),\phi(x')\big)$ for a Borel $H$. The pair's density ratio is then
$\sigma(\phi(X))\otimes\sigma(\phi(X'))$-measurable, whence
$\chi^2(\phi(X);\phi(X'))=\chi^2(X;X')$, hence $\norm{T_{\phi}}_{\HS}^2=\norm{T}_{\HS}^2$ by
\eqref{eq:objective-as-divergence}.}
By hypothesis $d^\ast=\overline{\dim}_{\mathrm{box}}(J(\mcX),\dJ)<\infty$, and a set of
finite box dimension admits a finite $\eps$-cover for every $\eps>0$, hence is totally
bounded. Since $\mcE$ is complete, $J(\mcX)$ is therefore precompact, and we may
replace it by its closure $M\subseteq\mcE$, a compact set of the same box dimension.
(Hilbert--Schmidt decay alone would not give this: the $\phi_j$ are orthonormal in
$L^2(\mu)$, not uniformly bounded, so $\norm{J(x)}^2=\chi^2(p(\cdot\cond x)\|\mu')+
\chi^2(q(\cdot\cond x)\|\mu)$ need not be bounded over $x$.)

Fix an integer $m>2d^\ast$. By the embedding theorem for compact
sets of finite box dimension \citep{sauer1991embedology,hunt1999regularity}, a prevalent
--- hence nonempty --- set of bounded linear maps $\pi:\mcE\to\RR^m$ is injective on $M$
with H\"older-continuous inverse on $\pi(M)$. Fix such a $\pi$ and put $\phi=\pi\circ J$.

\emph{Admissibility.} For $x,y\in\mcX$,
\begin{equation*}
  |\phi(x)-\phi(y)|=\big|\pi\big(J(x)-J(y)\big)\big|
  \le\norm{\pi}\,\norm{J(x)-J(y)}_{\mcE}=\norm{\pi}\,\dJ(x,y),
\end{equation*}
so $\phi$ is $\dJ$-Lipschitz, in particular continuous in the joint metric.

\emph{$\phi$ resolves the singular system.} Since $\pi$ is injective on $M\supseteq J(\mcX)$,
$J(x)=(\pi|_M)^{-1}(\phi(x))$, so every coordinate of $J$ --- each $\sigma_j\phi_j$ and each
$\sigma_j\psi_j$ --- is a continuous function of $\phi$. For $\sigma_j>0$ this makes $\phi_j$
measurable with respect to $\sigma(\phi(X))$, that is $\phi_j\in\Vphi$, and likewise
$\psi_j\in\Vphip$.

\emph{The full norm is attained.} Writing the singular decomposition
$T=\sum_{j}\sigma_j\inner{\psi_j}{\cdot}_{L^2(\mu')}\phi_j$ and using
$P_{\Vphi}\phi_j=\phi_j$, $P_{\Vphip}\psi_j=\psi_j$ for every $j$ with $\sigma_j>0$,
\begin{equation*}
  P_{\Vphi}\,T\,P_{\Vphip}
  =\sum_{j}\sigma_j\inner{P_{\Vphip}\psi_j}{\cdot}P_{\Vphi}\phi_j
  =\sum_{j}\sigma_j\inner{\psi_j}{\cdot}\phi_j=T ,
\end{equation*}
so $\norm{T_{\phi}}_{\HS}=\norm{T}_{\HS}$. Conversely
$\norm{P_{\Vphi}TP_{\Vphip}}_{\HS}\le\norm{T}_{\HS}$ for every $\phi$, since orthogonal
projections are contractions, so this $\phi$ attains the largest value the criterion can
take, $\norm{T_{\phi}}_{\HS}^2=\norm{T}_{\HS}^2$, which is \eqref{eq:bridge-bracket}.
In particular $\lfloor2d^\ast\rfloor+1$ coordinates suffice.

\section{Experimental Details and Reproducibility}
\label{sec:repro}

\paragraph{Compute.} All experiments run CPU-only on a single Apple M1 laptop
(macOS), in Python 3.11 with PyTorch \citep{paszke2019pytorch}. Library versions
are pinned in \texttt{requirements.txt}.

\paragraph{Seeds.} Data seeds are fixed integer literals: $1$ for the $\phi$
training pairs, $2$ for the model-selection split, and $10/11/12$ for the readout
train/validation/test pools. Model-initialization seeds are derived formulas,
$311+7i$ over the $\phi$ restarts and $17+991\,r$ over the VAMPnet restarts, and
the warp maps carry their own fixed seeds ($7$, $11$, $99$). The dataset is
therefore identical across what we report as ``seeds'', so the seed spread
measures optimizer-initialization variability and not sampling variability. The
sole sampling-side error bar in the paper is the resampled label draws of
\Cref{subsec:exp-fewshot}. No torch determinism flags are set. Results are
float32 on CPU and machine-stable, but not guaranteed bitwise-portable across
BLAS builds.

\paragraph{Metrics.} Every recovery number is a per-component test $R^2$:
writing $y$ for one target coordinate over the held-out test pool, $\hat y$ for
its prediction from the learned features, and $\bar y$ for the mean of $y$ on
that pool, $R^2=1-\sum(y-\hat y)^2/\sum(y-\bar y)^2$, computed per coordinate
and averaged within a component (two coordinates for a cool, one for a bit). It
is not clipped, so a representation that predicts a component worse than that
component's own mean scores negative, and a component the representation has
dropped reads $0$ up to readout noise. Span-side selection uses VAMP-E
\citep{wu2020variational}, which at a rank-$k$ feature pair takes the value
$2\operatorname{tr}(U^\top C_{01}V)-\operatorname{tr}(U^\top C_{00}U\,V^\top
C_{11}V)$ with the whitening transforms $U,V$ fitted on the training split and
the feature covariances --- $C_{00}$ and $C_{11}$ instantaneous on the present and
future features, $C_{01}$ across the lag --- formed on a disjoint one, so it is an
unbiased estimate of a population score that overfitting can only depress.
\Cref{fig:capacity} reports the \emph{fraction built} of a singular mode,
$R^2/\sigma^2$, where $R^2$ is that of a linear readout of the mode's future
value from the present features and $\sigma^2$ is the population $R^2$ available
for a unit-variance singular function of singular value $\sigma$: one means a
fully represented mode, zero an absent one.

\paragraph{Hyperparameter searches.} Beyond the settings of record, the
following ranges were tried during development. Span: trunk widths $64^2$,
$128^2$, $256^2$, $512^2$ at ranks $k\in\{12,30,50,80,100,120\}$. Algebra:
embedding dimension $m\in\{1,2,4,6,8,10,12,14,16,18,20\}$, critic learning rate
$2\cdot10^{-4}$ to $2\cdot10^{-3}$ in four steps, with held-out $\hat J$
increasing across the range so the largest was kept, and critic weight decay
$0$ or $10^{-4}$, the latter never selected. On the warped observation we swept
the $\phi$ learning rate over $10^{-3}$, $3\cdot10^{-4}$, $10^{-4}$ both cold
(held-out $\hat J$ of $93$, $116$, $86$) and warm-started ($140$, $161$,
$173$), keeping $3\cdot10^{-4}$, and the $\phi$ trunk over $128^2$, $256^2$,
$512^2$, $128^3$, $128^4$, every widening or deepening of which lowered $\hat J$
($93$, $79$, $62$, $75$, $60$), so the warped $\phi$ stays at $128^2$. The
readout tunes learning rate $\{3\cdot10^{-3},10^{-2}\}$ against weight decay
$\{10^{-5},10^{-3}\}$ on its own validation split. Restarts and stopping epochs
were selected on held-out VAMP-E for the span and on held-out $\hat J$ for the
algebra, and no hyperparameter anywhere was selected on a test readout.

\end{document}